\documentclass[10pt]{article}
\usepackage[utf8]{inputenc} % allow utf-8 input
\usepackage[T1]{fontenc}    % use 8-bit T1 fonts
\usepackage{url}            % simple URL typesetting
\usepackage{booktabs}       % professional-quality tables
\usepackage{amsfonts}       % blackboard math symbols
\usepackage{nicefrac}       % compact symbols for 1/2, etc.
\usepackage{microtype}      % microtypography
\usepackage[table]{xcolor}
\usepackage{natbib}
\setcitestyle{authoryear,open={[},close={]},semicolon,aysep={,},yysep={;}}
\usepackage{fullpage}
\usepackage{multirow}
\usepackage{soul}         
\usepackage{amsmath}% delete line
\allowdisplaybreaks
\newcommand{\calA}{\mathcal{A}}

\newcommand{\calS}{\mathcal{S}}
\newcommand{\calT}{\mathcal{T}}
\newcommand{\calF}{\mathcal{F}}

\newcommand{\E}{\mathbb{E}}

\newcommand{\ls}{\left}
\newcommand{\rs}{\right}

\newcommand{\interior}{\mathrm{int}}
\newcommand{\rint}{\mathrm{rint}}
\newcommand{\dom}{\mathrm{dom}}

\makeatletter

\newcommand{\Rmnum}[1]{\textup{\expandafter\@slowromancap\romannumeral #1@}}
\makeatother

\usepackage{tabstackengine}
\usepackage{appendix}
\usepackage{titletoc}
\usepackage{titlesec}
\usepackage{enumerate}
\usepackage{amsmath}
\usepackage{amssymb}
\usepackage{mathtools}
\usepackage{amsthm}
\usepackage{algorithm}
\usepackage{algorithmic}
\usepackage{graphicx}
\usepackage{placeins}
\usepackage{subfigure}
\usepackage{makecell}
\usepackage{array}
\usepackage{threeparttable}
\usepackage{setspace}
\usepackage{dsfont}
\usepackage{multicol}
\usepackage{authblk}
\usepackage{lipsum}
\usepackage{tabularx}
\usepackage{enumitem}
\definecolor{citationblue}{RGB}{20,75,135}
\usepackage[
    colorlinks=true,
    linkcolor=black,
    citecolor=citationblue,
    urlcolor=citationblue
]{hyperref}                            % hyperlinks
\allowdisplaybreaks

\theoremstyle{plain}
\newtheorem{theorem}{Theorem}[section]

\newtheorem{lemma}[theorem]{Lemma}
\newtheorem{corollary}[theorem]{Corollary}

\theoremstyle{definition}

\newtheorem{assumption}{Assumption}[section]

\theoremstyle{remark}

\title{\textbf{Fast Regularized Policy Mirror Descent with One-Step TD Updates}}
\author[1]{Qipei Chen}
\author[2]{Wenye Li}
\author[2]{Yule Sun}
\author[2]{Ke Wei}
\affil[1]{School of Mathematical Sciences, Fudan University}
\affil[2]{School of Data Science, Fudan University}

\begin{document}

\maketitle

\begin{abstract}
Policy mirror descent (PMD) enjoys fast convergence in regularized Markov decision processes (MDPs), but existing guarantees often rely on exact or increasingly accurate policy evaluation. We analyze PMD coupled with a persistent critic advanced by one temporal-difference (TD) update. For finite discounted MDPs, we establish global linear convergence in value for exact coordinate-wise Bellman updates, with any positive constant actor stepsize and arbitrary finite critic initialization. The proof combines a resolvent-based auxiliary distribution with a decaying Bellman-violation correction and a potential weighted by inverse coordinate weights. We then study stochastic TD--PMD with general strongly convex mirror maps under a single off-policy Markov trajectory. With suitably chosen constant stepsizes and a finite-batch TD update, the method achieves an expected value gap of $\epsilon$ after $\widetilde{\mathcal{O}}\!\left(\frac{1}{(1-\gamma)^5\widetilde{\sigma}_b\epsilon}\right)$ transitions. The stochastic analysis relies on the trajectory-wise Lipschitz continuity of the regularizer, derived from uniform bounds on vertex Bregman divergences, together with a visitation-weighted resolvent estimate for signed critic-error propagation that yields an inverse-linear dependence on behavior coverage $\widetilde{\sigma}_b$.  In contrast to many prior guarantees for regularized policy optimization, our sample-complexity guarantee holds without trajectory resets, generative-model access, or nested policy-evaluation loops. Numerical results are consistent with the theoretical convergence analysis.
\end{abstract}

\section{Introduction}
\label{sec:introduction}

Regularization is a fundamental design principle in reinforcement learning (RL).  In large-scale problems with complex dynamics, maintaining sufficient randomness in the policy iterates helps sustain exploration and reduces the risk of committing too early to a suboptimal policy \citep{Husain_Ciosek_Tomioka_2021}.  A variety of policy regularizers have been studied in RL, including those based on Shannon entropy, Kullback--Leibler divergence, the squared $\ell_2$ norm, and Tsallis entropies \citep{Geist_Scherrer_Pietquin_2019,Lan_2021, Chow_Nachum_Ghavamzadeh_2018,Lee_Choi_Oh_2018, Lee_Kim_Lim_Choi_Oh_2019}.  Regularizers can also promote sparse or structured policies and incorporate information from a reference policy. Therefore, regularized formulations are very useful for robust, cost-constrained, and safety-aware decision making \citep{Husain_Ciosek_Tomioka_2021,Lan_2021,Ying_Ding_Lavaei_2022}.

A general optimization framework for regularized RL is the regularized form of policy mirror descent (PMD). In this formulation, PMD improves the policy through a statewise Bregman-proximal step and, depending on the choice of mirror geometry, includes softmax natural policy gradient (NPG) and projected Q-ascent (PQA) as representative instances corresponding to KL and Euclidean geometries, respectively \citep{kakade2002npg,Xiao_2022,Lan_2021}. Entropy-regularized NPG and regularized PMD enjoy global geometric convergence \citep{Cen_Cheng_Chen_Wei_Chi_2022,Lan_2021}. The framework has since been extended in several directions. Generalized PMD (GPMD) accommodates broad classes of convex, possibly nonsmooth regularizers and controlled evaluation and update errors \citep{Zhan_Cen_Huang_Chen_Lee_Chi_2021}. Homotopic, lookahead, implicit-exploration, policy-convergence, and strongly-polynomial variants further extend PMD in terms of convergence guarantees and computational properties \citep{Li_Zhao_Lan_2022,protopapas2024policy,li2025policy,LiWei2026PolicyConvergence,JuLan2026}. For strongly convex mirror maps, regularized stochastic PMD can attain an $\widetilde{\mathcal O}(\epsilon^{-1})$ dependence on the target accuracy \citep{Lan_2021}. In this paper, we focus on regularized PMD. For the algorithmic and theoretical developments of unregularized PMD or its specific instances, see~\citet{Xiao_2022,Johnson_Pike-Burke_Rebeschini_2023,Agarwal_Kakade_Lee_Mahajan_2019,pg-liu,ppgliu,Lin2022PMD-policy-convergence,Khodadadian_Jhunjhunwala_Varma_Maguluri_2021,yuan2023general,Alfano_Rebeschini_2022,chelu2024functional,feng2024global} and references therein.

Existing analyses of regularized PMD, however, generally rely on exact or sufficiently accurate evaluation of the current policy before each policy improvement step. For example, exact analyses assume access to the current policy's action values \citep{Cen_Cheng_Chen_Wei_Chi_2022,Lan_2021, Johnson_Pike-Burke_Rebeschini_2023}, while inexact analyses assume an evaluation oracle that returns action-value estimates with controlled error \citep{Zhan_Cen_Huang_Chen_Lee_Chi_2021}. Under generative-model access, samples are used at each policy iteration to estimate the current action values to a prescribed accuracy \citep{Lan_2021,Li_Zhao_Lan_2022, Johnson_Pike-Burke_Rebeschini_2023}; under Markovian sampling, a temporal-difference or Monte Carlo evaluator is run for a fixed current policy before the next policy improvement step \citep{Lan_2021,li2025policy}. Thus, across these settings, the critic is required to track the action-value function of the current policy sufficiently accurately before the next actor update is performed.

In contrast, temporal-difference (TD) learning provides a natural way to update the critic incrementally from its current estimate, without solving the Bellman equation for each newly updated policy to a prescribed accuracy before the next policy update \citep{sutton1988td,suttonRL}. A related idea appears in regularized modified policy iteration \citep{Geist_Scherrer_Pietquin_2019}, where only finitely many Bellman evaluation steps are performed between successive policy improvements. With a single evaluation step, the resulting recursion is closely related to one-step TD--PMD. More recently, one-step TD--PMD has been analyzed for unregularized MDPs in the exact and generative-model settings \citep{liu2025tdpmd}, and subsequently under online Markov data \citep{Li2026Markov-TDPMD}. In the exact setting, under a one-sided Bellman initialization, a monotonicity property of TD--PMD with constant stepsizes is established and used to derive last-iterate sublinear convergence, and a shifting argument then extends the result to arbitrary initializations \citep{liu2025tdpmd}. Regularized settings have also been studied with single-loop methods that do not require each intermediate policy to be evaluated to a prescribed accuracy. For entropy-regularized zero-sum Markov games, a single-loop actor--critic method uses optimistic NPG for the policy step and updates the value estimate at each iteration \citep{cen2022faster}. Another related approach for regularized MDPs is value mirror descent (VMD), which takes a value-iteration approach: it computes a mirror-based policy update from the current value estimate and then applies a Bellman value update under the resulting policy \citep{jia2026value}. In the exact setting, linear convergence of VMD is established with an increasing, epoch-dependent stepsize schedule for one-sided Bellman initialization. The proof also uses a monotonicity property of the value iterates, similar to the  analysis developed in \citet{liu2025tdpmd}.

Despite these developments, for standard regularized PMD with only one TD update per policy step, a fast convergence rate in the exact setting remains unavailable, while the corresponding convergence theory under Markov data is still lacking. Existing one-step TD--PMD results focus on unregularized MDPs, covering the exact, generative-model, and online Markov-data settings \citep{liu2025tdpmd,Li2026Markov-TDPMD}. In regularized settings, the single-loop method of \citet{cen2022faster} is developed specifically for the entropy regularizer corresponding to NPG, and its discounted analysis is full-information, requires sufficiently small actor stepsizes, and uses a full-support evaluation distribution. Value mirror descent (VMD) \citep{jia2026value} considers regularized MDPs from a different algorithmic perspective: it replaces the classical PMD recursion with an epochwise value-iteration scheme and achieves linear convergence in the exact setting with increasing, epoch-dependent policy stepsizes under a one-sided Bellman initialization. Its stochastic variant, SVMD, further relies on generative-model samples to construct increasingly accurate estimates of the transition kernel and performs variance-reduced Bellman updates on the estimated models, making it model-based rather than model-free TD.  Thus, a general convergence theory for standard regularized PMD with one-step TD updates is still missing. The online setting is further complicated by the need to track a changing target from a single off-policy Markov trajectory. This gap motivates the central question of this paper:

\begin{quote}
\emph{Can the fast convergence of standard regularized PMD be retained when the critic is updated by only one TD step, even along a single  off-policy Markov trajectory?}
\end{quote}

We answer this question affirmatively. The main contributions of this paper are summarized as follows. See Table~\ref{tab:closest-related-work} for a comparison with existing regularized PMD results.
\begin{itemize}[leftmargin=*,itemsep=0.45em]

\item \textbf{Linear convergence for regularized TD--PMD with coordinate-wise critic updates.}
We study regularized TD--PMD with coordinate-wise Bellman updates, allowing any fixed diagonal update matrix $W$ satisfying $0<W\le I$. For a general convex mirror map, we establish its global linear convergence in value. When the mirror map is strongly convex, policy convergence is further obtained. The results allow arbitrary finite critic initialization, any admissible initial policy, every positive constant actor stepsize, and any state distribution used to evaluate performance, including distributions without full support. The analysis relies on a decaying Bellman-violation correction for arbitrary critic initialization, a change-of-distribution argument that avoids full-support requirements, and a potential function with inverse coordinate weights for nonuniform critic updates.

\item \textbf{Stochastic convergence under a single off-policy Markov trajectory.}
With a constant actor stepsize and a finite-batch TD update per policy step, regularized TD--PMD with a strongly convex mirror map achieves an expected value gap of at most $\epsilon$ after
\[
\widetilde{\mathcal O}\!\left(
\frac{1}{(1-\gamma)^5\widetilde\sigma_b\epsilon}
\right)
\]
observed transitions, where $\widetilde\sigma_b>0$ is the smallest stationary state--action visitation probability under the behavior policy. The method uses neither trajectory resets nor generative-model access and requires no nested policy-evaluation loop. A key step in the analysis is to establish trajectory-wise Lipschitz continuity of the regularizer from uniform bounds on vertex Bregman divergences, which in turn controls the policy drift and the variation of $Q_\tau^{\pi_k}$. We then use a signed critic-error decomposition and geometric output weighting to handle the resulting moving-target and Markov-sampling errors.
\end{itemize}

%Numerical experiments on random finite MDPs complement the theory.  They
%illustrate the predicted linear convergence of full and behavior-weighted
%deterministic updates, including from a critic initialization that violates the
%usual one-sided Bellman condition, and the behavior of the persistent critic
%under continuous off-policy Markov sampling.

%{\color{red}Table~\ref{tab:closest-related-work} compares representative PMD-type results in terms of regularization and how the action-value  is evaluated.}

\begin{table}[ht!]
\centering
\caption{Comparison of regularized PMD results. Here, $\widetilde\sigma_b$ is the smallest stationary visitation probability over state--action pairs; under near-uniform coverage, $\widetilde\sigma_b=\Theta((|\calS||\calA|)^{-1})$.}

\vspace{0.2cm}

\label{tab:closest-related-work}
\scriptsize
\setlength{\tabcolsep}{1.4pt}
\renewcommand{\arraystretch}{1.15}

\begin{tabularx}{\textwidth}{@{}
>{\raggedright\arraybackslash}p{0.145\textwidth}
>{\centering\arraybackslash}p{0.145\textwidth}
>{\centering\arraybackslash}p{0.31\textwidth}
>{\centering\arraybackslash}p{0.205\textwidth}
>{\centering\arraybackslash}X@{}}
\toprule

&
\textbf{Algorithm} &
\textbf{Iteration complexity} &
\textbf{Stochastic access} &
\textbf{Sample complexity} \\

\midrule

\citet{Cen_Cheng_Chen_Wei_Chi_2022} &
NPG &
$\mathcal O\!\left(
\log_{(1-\eta\tau)^{-1}}(\epsilon^{-1})
\right)$ &
--- &
--- \\

\citet{cen2022faster} &
TD--NPG &
$\mathcal O\!\left(
\log_{[1-\frac{(1-\gamma)\eta\tau}{4}]^{-1}}
(\epsilon^{-1})
\right)$ &
--- &
--- \\

\citet{Zhan_Cen_Huang_Chen_Lee_Chi_2021} &
PMD &
$\mathcal O\!\left(
\log_{[1-\frac{\eta\tau}{1+\eta\tau}(1-\gamma)]^{-1}}
(\epsilon^{-1})
\right)$ &
--- &
--- \\

\citet{Lan_2021} &
PMD &
$\mathcal O\!\left(
\log_{\gamma^{-1}}(\epsilon^{-1})
\right)$ &
\makecell[c]{Generative model} &
$\widetilde{\mathcal O}\!\left(
\frac{|\calS||\calA|}
{(1-\gamma)^5\epsilon}
\right)$ \\

\citet{jia2026value} &
VMD &
$\mathcal O\!\left(
(1-\gamma)^{-1}
\log_2([(1-\gamma)\epsilon]^{-1})
\right)$ &
\makecell[c]{Generative model} &
$\widetilde{\mathcal O}\!\left(
\frac{|\calS||\calA|}
{(1-\gamma)^5\epsilon}
\right)$ \\

\textbf{This work} &
\textbf{TD--PMD} &
$\mathcal O\!\left(
\log_{
[1 - \underline w \min\{1- \gamma,\frac{\eta\tau}{1+\eta\tau}\}]^{-1}
}
(\epsilon^{-1})
\right)$ &
\makecell[c]{\textbf{Off-policy}\\\textbf{Markov data}} &
$\widetilde{\mathcal O}\!\left(
\frac{1}
{(1-\gamma)^5\widetilde\sigma_b\epsilon}
\right)$ \\

\bottomrule
\end{tabularx}

% \begin{minipage}{\textwidth}
% \scriptsize
% \emph{Complexity convention.}
% Both result columns report the complexity of reaching $\epsilon$-accuracy. Exact entries report iteration complexity, whereas stochastic entries report sample complexity, measured by the number of transition samples.
% For stochastic results, the strongly convex stochastic case is displayed. The $\widetilde{\mathcal O}$ notation suppresses logarithmic factors and paper-specific fixed constants.  The stochastic guarantees may differ in their output rules and probability criteria.  Here, $\widetilde\sigma_b$ is the smallest stationary visitation probability over state--action pairs; under near-uniform coverage, $\widetilde\sigma_b=\Theta((|\calS||\calA|)^{-1})$.
% \end{minipage}
\end{table}

The rest of the paper is organized as follows. Section~\ref{sec:problem-setting} introduces the problem setup and necessary preliminaries. Section~\ref{sec:exact-diagonal-td-pmd} studies regularized TD--PMD with deterministic coordinate-wise Bellman updates. Section~\ref{sec:off-policy-markov-data} extends the analysis to off-policy Markov data. Section~\ref{sec:numerical-experiments} presents numerical results that are consistent with the theoretical convergence analysis, and Section~\ref{sec:conclusion} concludes the paper with directions for future work. The proofs of the main results and technical lemmas are deferred to the appendices.

\section{Problem Setting and Preliminaries}
\label{sec:problem-setting}

\subsection{Regularized MDPs and Bellman preliminaries}
\label{sec:regularized-mdp-bellman}

\paragraph{Model and objective.}
Consider a finite discounted Markov decision process (MDP) $\mathcal M=(\calS,\calA,P,r,\gamma)$, where $\calS$ and $\calA$ are finite state and action spaces, $P(\cdot\mid s,a)$ is the transition kernel, $r(s,a)$ is the immediate reward, and $\gamma\in[0,1)$ is the discount factor. Letting $\Delta(\calA)$ be the probability simplex over $\calA$, the space of stationary Markov policies is
\[
\Pi:=\{\pi:\pi(\cdot\mid s)\in\Delta(\calA),\ s\in\calS\}.
\]
Under a policy $\pi$, a trajectory evolves according to $a_t\sim\pi(\cdot\mid s_t)$ and $s_{t+1}\sim P(\cdot\mid s_t,a_t)$.

\begin{assumption}[Reward and regularization parameters]
\label{ass:normalization}
For simplicity, assume $0\leq r(s,a)\leq1$ for every $(s,a)\in\calS\times\calA$. Unless otherwise specified, we also assume $\tau>0$ throughout the paper. 
\end{assumption}

Let $h:\mathbb R^{|\calA|}\to\mathbb R\cup\{+\infty\}$ be a convex policy regularizer and write $h^\pi(s):=h(\pi(\cdot\mid s))$.  For a policy $\pi$, its regularized state- and action-value functions are
\begin{align*}
V_\tau^\pi(s)
&:=\E\!\left[\sum_{t=0}^\infty\gamma^t
\bigl(r(s_t,a_t)-\tau h^\pi(s_t)\bigr)\,\middle|\,s_0=s\right],\\
Q_\tau^\pi(s,a)
&:=\E\!\left[r(s_0,a_0)+\sum_{t=1}^\infty\gamma^t
\bigl(r(s_t,a_t)-\tau h^\pi(s_t)\bigr)
\,\middle|\,(s_0,a_0)=(s,a)\right].
\end{align*}
It follows directly from the definitions that
\[
V_\tau^\pi(s)
=\E_{a\sim\pi(\cdot\mid s)}[Q_\tau^\pi(s,a)]-\tau h^\pi(s),
\qquad
Q_\tau^\pi(s,a)
=r(s,a)+\gamma\E_{s'\sim P(\cdot\mid s,a)}[V_\tau^\pi(s')].
\]
For $\mu\in\Delta(\calS)$, we abbreviate
$V_\tau^\pi(\mu):=\E_{s\sim\mu}[V_\tau^\pi(s)]$.

\paragraph{Bellman representation.}
The  Bellman operators acting on state- and action-value vectors are
\begin{align*}
\calT_\tau^\pi V(s)
&:=\E_{\substack{a\sim\pi(\cdot\mid s)\\s'\sim P(\cdot\mid s,a)}}
[r(s,a)+\gamma V(s')-\tau h^\pi(s)],\\
\calF_\tau^\pi Q(s,a)
&:=\E_{\substack{s'\sim P(\cdot\mid s,a)\\a'\sim\pi(\cdot\mid s')}}
[r(s,a)+\gamma Q(s',a')-\tau\gamma h^\pi(s')].
\end{align*}
Equivalently,
\begin{align*}
\calT_\tau^\pi V
&=r^\pi+\gamma P_{\scriptscriptstyle\calS}^\pi V
-\tau h_{\scriptscriptstyle\calS}^\pi,\\
\calF_\tau^\pi Q
&=r+\gamma P_{\scriptscriptstyle\calS\times\calA}^\pi Q
-\tau\gamma h_{\scriptscriptstyle\calS\times\calA}^\pi,
\end{align*}
where $r^\pi(s):=\sum_a\pi(a\mid s)r(s,a)$,
$h_{\scriptscriptstyle\calS}^\pi(s):=h^\pi(s)$, 
$h_{\scriptscriptstyle\calS\times\calA}^\pi(s,a)
:=\sum_{s'}P(s'\mid s,a)h^\pi(s')$. The state and state--action transition matrices are defined by
\begin{align*}
P_{\scriptscriptstyle\calS}^\pi(s,s')
&:=\sum_{a\in\calA}\pi(a\mid s)P(s'\mid s,a),\\
P_{\scriptscriptstyle\calS\times\calA}^\pi((s,a),(s',a'))
&:=P(s'\mid s,a)\pi(a'\mid s').
\end{align*}

The following standard properties, which can be verified directly, identify policy evaluation with a contractive fixed-point problem.
\begin{lemma}[Bellman fixed points and contractions]
\label{lem:Bellman-operators}
For any $\pi\in\Pi$, the regularized values are the unique fixed points
\[
\calT_\tau^\pi V_\tau^\pi=V_\tau^\pi,
\qquad
\calF_\tau^\pi Q_\tau^\pi=Q_\tau^\pi.
\]
Moreover, $\calT_\tau^\pi$ and $\calF_\tau^\pi$ are monotone and are $\gamma$-contractions in the supremum norm: for all conformable $V,V'$ and $Q,Q'$,
\[
\|\calT_\tau^\pi V-\calT_\tau^\pi V'\|_\infty
\leq\gamma\|V-V'\|_\infty,
\qquad
\|\calF_\tau^\pi Q-\calF_\tau^\pi Q'\|_\infty
\leq\gamma\|Q-Q'\|_\infty.
\]
\end{lemma}

The optimal Bellman operators are 
\begin{align*}
\calT_\tau V(s)
&:=\max_{p\in\Delta(\calA)}
\left\{\sum_a p(a)\left[r(s,a)+\gamma
\E_{s'\sim P(\cdot\mid s,a)}V(s')\right]-\tau h(p)\right\},\\
\calF_\tau Q(s,a)
&:=r(s,a)+\gamma\E_{s'\sim P(\cdot\mid s,a)}
\left[\max_{p\in\Delta(\calA)}
\{\langle p,Q(s',\cdot)\rangle-\tau h(p)\}\right].
\end{align*}
Under the mirror-map conditions below, an optimal stationary policy exists for the finite discounted MDP \citep{VI-PI,Geist_Scherrer_Pietquin_2019}.  We fix one such policy $\pi_\tau^*$ and write
\[
V_\tau^*(s):=V_\tau^{\pi_\tau^*}(s)
=\max_{\pi\in\Pi}V_\tau^\pi(s),
\qquad
Q_\tau^*(s,a):=Q_\tau^{\pi_\tau^*}(s,a)
=\max_{\pi\in\Pi}Q_\tau^\pi(s,a).
\]

\paragraph{Distributional notation and policy comparison.}
A stationary distribution of $P_{\scriptscriptstyle\calS}^\pi$, denoted $\nu^\pi$, is any distribution satisfying $(\nu^\pi)^\top P_{\scriptscriptstyle\calS}^\pi=(\nu^\pi)^\top$. The associated stationary state--action distribution is $\sigma^\pi(s,a):=\nu^\pi(s)\pi(a\mid s)$, which satisfies $(\sigma^\pi)^\top P_{\scriptscriptstyle\calS\times\calA}^\pi=(\sigma^\pi)^\top$.  For an initial distribution $\mu\in\Delta(\calS)$, the  discounted state-visitation measure is
\[
d_\mu^\pi(s):=(1-\gamma)\E\!\left[
\sum_{t=0}^\infty\gamma^t\mathds 1\{s_t=s\}
\,\middle|\,s_0\sim\mu\right].
\]

For the fixed optimal policy $\pi_\tau^*$, write
$d_\mu^*:=d_\mu^{\pi_\tau^*}$, and let $\nu^*$ denote
any stationary distribution of
$P_{\scriptscriptstyle\calS}^{\pi_\tau^*}$.
We do not assume that $\nu^*$ is unique or has full support.

The following performance difference lemma converts a policy comparison into an average one-step Bellman advantage.

\begin{lemma}[Performance difference lemma~\citep{kakade2002approximately}]
\label{lem:pdl}
For any $\pi,\pi'\in\Pi$ and $\mu\in\Delta(\calS)$,
\begin{align*}
V_\tau^\pi(\mu)-V_\tau^{\pi'}(\mu)
&=\frac{1}{1-\gamma}\E_{s\sim d_\mu^\pi}
\bigl[\calT_\tau^\pi V_\tau^{\pi'}(s)-V_\tau^{\pi'}(s)\bigr]\\
&=\frac{1}{1-\gamma}\E_{s\sim d_\mu^\pi}\!\left[
\langle\pi(\cdot\mid s)-\pi'(\cdot\mid s),
Q_\tau^{\pi'}(s,\cdot)\rangle
-\tau\bigl(h^\pi(s)-h^{\pi'}(s)\bigr)\right].
\end{align*}
\end{lemma}

\subsection{Mirror geometry and policy mirror descent}
\label{sec:mirror-pmd}

\paragraph{Mirror map and standing conditions.} The following condition is standard for  mirror maps.
\begin{assumption}[Mirror map]
\label{ass:h}
The function
$h:\mathbb R^{|\calA|}\to\mathbb R\cup\{+\infty\}$ is proper, closed,
and convex, with $\dom h\supseteq[0,1]^{|\calA|}$.  It is differentiable
on $\interior\dom h$ and essentially smooth: if
$p_n\in\interior\dom h$ and
$p_n\to p\in\mathrm{bd}\,\dom h$, then
$\|\nabla h(p_n)\|_\infty\to\infty$.
\end{assumption}

Essential smoothness ensures that a Bregman-proximal update initialized in $\rint(\dom h)$ remains in this set.  This boundary condition is automatically satisfied when $\dom h=\mathbb R^{|\calA|}$, since the domain has no boundary, as in the quadratic mirror map used by PQA. The negative-entropy mirror map $h(p)=\sum_a p(a)\log p(a)$ on $[0,\infty)^{|\calA|}$ is also essentially smooth because $1+\log p(a)\to-\infty$ as $p(a)\downarrow0$.

Assumption~\ref{ass:h} implies that $h$ is bounded on the simplex. Indeed, since $h$ is closed, it attains a finite minimum on the compact set $\Delta(\calA)$, while convexity gives
\[
h(p)\leq\sum_a p(a)h(e_a)\leq\max_a h(e_a),
\]
where $e_a$ is the $a$-th vertex of the simplex.  Hence, throughout the paper, we set
\[
H_h:=\max_{p\in\Delta(\calA)}|h(p)|<\infty.
\]
It follows immediately that all regularized value functions admit uniform bounds.
\begin{lemma}[Bounded regularized values]
\label{lem:bounded_value}
Under Assumptions~\ref{ass:normalization} and~\ref{ass:h}, for every
$\pi\in\Pi$,
\[
|V_\tau^\pi(s)|\leq\frac{1+\tau H_h}{1-\gamma},
\qquad
|Q_\tau^\pi(s,a)|\leq
\frac{1+\tau\gamma H_h}{1-\gamma}.
\]
\end{lemma}

Policy-distance guarantees and the stochastic analysis additionally use the following condition.
\begin{assumption}[Strong convexity]
\label{ass:strong-convexity}
The mirror map $h$ is $\lambda$-strongly convex on $\Delta(\calA)$ with respect to the $\ell_1$-norm for some $\lambda>0$.
\end{assumption}

\paragraph{Bregman geometry and the PMD update.}
For $q\in\rint(\dom h)$, let
\[
D_h(p\,\|\,q):=h(p)-h(q)-\langle\nabla h(q),p-q\rangle
\]
be the Bregman divergence generated by $h$.  We use the statewise and distribution-averaged notation
\[
D_{\pi'}^\pi(s)
:=D_h(\pi(\cdot\mid s)\,\|\,\pi'(\cdot\mid s)),
\qquad
D_{\pi'}^\pi(\mu)
:=\E_{s\sim\mu}[D_{\pi'}^\pi(s)].
\]

\begin{assumption}[Policy initialization]
\label{ass:policy-initialization}
Every PMD scheme considered below is initialized with a target policy satisfying
\[
\pi_0(\cdot\mid s)\in\rint(\dom h)\cap\Delta(\calA),
\qquad s\in\calS.
\]
\end{assumption}
For quadratic geometry, Assumption~\ref{ass:policy-initialization} is automatic because $\dom h=\mathbb R^{|\calA|}$.  For negative entropy, it requires a full-support initial target policy. Given a critic $Q_\tau^k$ and a policy stepsize $\eta >0$, PMD updates every state according to
\begin{align}
\pi_{k+1}(\cdot\mid s)
\in\arg\max_{p\in\Delta(\calA)}
\left\{\langle p,Q_\tau^k(s,\cdot)\rangle-\tau h(p)
-\frac1\eta D_h(p\,\|\,\pi_k(\cdot\mid s))\right\}.
\label{eq:PMD-policy-update}
\end{align}

Note that  a maximizer exists under Assumption~\ref{ass:h}. Moreover, when Assumption~\ref{ass:strong-convexity} holds, the maximizer is unique, and otherwise $\pi_{k+1}(\cdot\mid s)$ denotes any maximizer.  With negative entropy, $D_h$ is the Kullback--Leibler divergence and the update recovers the geometry of entropy-regularized natural policy gradient \citep{kakade2002npg,Cen_Cheng_Chen_Wei_Chi_2022}; with a quadratic $h$, it is a Euclidean proximal policy update.

Classical PMD typically constructs $Q_\tau^k$ by exact or controlled-inexact policy evaluation.  In this paper, $Q_\tau^k$ is instead a TD critic and need not equal $Q_\tau^{\pi_k}$. The critic therefore evolves together with the policy sequence, rather than serving as an exact evaluator of each current policy. The main technical tool is the three-point inequality associated with the statewise PMD subproblem.
\begin{lemma}[Three-point inequality \citep{Xiao_2022,Lan_2021,jia2026value}]
\label{lem:three-point-descent}
Suppose Assumption~\ref{ass:h} holds and $\pi_k(\cdot\mid s)\in\rint(\dom h)\cap\Delta(\calA)$.  Then the update \eqref{eq:PMD-policy-update} satisfies $\pi_{k+1}(\cdot\mid s)\in\rint(\dom h)\cap\Delta(\calA)$ and, for every $p\in\Delta(\calA)$,
\[
\langle\pi_{k+1}(\cdot\mid s)-p,Q_\tau^k(s,\cdot)\rangle
-\tau\bigl(h^{\pi_{k+1}}(s)-h(p)\bigr)
\geq\eta^{-1}\!\left[D_{\pi_k}^{\pi_{k+1}}(s)
+(1+\eta\tau)D_{\pi_{k+1}}^p(s)-D_{\pi_k}^p(s)\right],
\]
where $D_\pi^p(s):=D_h(p\,\|\,\pi(\cdot\mid s))$.
\end{lemma}

\section{Exact TD--PMD with Coordinate-Wise Bellman Updates}
\label{sec:exact-diagonal-td-pmd}

Regularized TD--PMD with coordinate-wise critic updates is given in Algorithm~\ref{alg:exact-diagonal-td-pmd}. At each iteration, after the policy update, the critic is advanced by a single Bellman update rather than being evaluated to its fixed point. The weight matrix $W$ allows different update magnitudes across state--action coordinates and satisfies
\begin{equation}
W:=\operatorname{diag}\!\left(w(s,a):(s,a)\in\calS\times\calA\right),
\qquad
0<\underline w:=\min_{s,a}w(s,a)
\leq \overline w:=\max_{s,a}w(s,a)\leq1.
\label{eq:diagonal-relaxation-assumption}
\end{equation}
It is evident that $W=I$ corresponds to the full one-step Bellman update, whereas a general diagonal $W$ allows coordinate-wise partial updates. Moreover, since $W$ is invertible, for every fixed policy $\pi$,
\begin{equation}
Q=Q+W(\calF_\tau^\pi Q-Q)
\quad\Longleftrightarrow\quad
\calF_\tau^\pi Q=Q.
\label{eq:diagonal-relaxation-fixed-point}
\end{equation}
Thus coordinate-wise weighting changes the dynamics and the rate, but not the policy-evaluation fixed point $Q_\tau^\pi$ or the optimal fixed point
$Q_\tau^*$.

\begin{algorithm}[ht!]
\small
\caption{Exact Regularized TD--PMD with Coordinate-Wise Bellman Updates}
\label{alg:exact-diagonal-td-pmd}
\begin{algorithmic}
\STATE {\bfseries Input:} Iterations $K$, initial finite critic $Q_\tau^0$, policy $\pi_0$, policy stepsize $\eta>0$, and diagonal update matrix $W$ satisfying \eqref{eq:diagonal-relaxation-assumption}.
\FOR{$k=0,1,\ldots,K-1$}
\STATE \textbf{(Policy update)} For every $s\in\calS$, set
\[
\pi_{k+1}(\cdot\mid s)\in\underset{p\in\Delta(\calA)}{\arg\max}
\left\{\langle p,Q_\tau^k(s,\cdot)\rangle-\tau h(p)
-\eta^{-1}D_h\!\left(p\,\middle\|\,\pi_k(\cdot\mid s)\right)\right\}.
\]
\STATE \textbf{(Critic update)} Set
\begin{equation*}
Q_\tau^{k+1}=Q_\tau^k+W\left(\calF_\tau^{\pi_{k+1}}Q_\tau^k-Q_\tau^k\right).
\end{equation*}
\ENDFOR
\STATE {\bfseries Output:} Last-iterate policy $\pi_K$ and critic $Q_\tau^K$.
\end{algorithmic}
\end{algorithm}

The value-convergence guarantee of Algorithm~\ref{alg:exact-diagonal-td-pmd} is stated in Theorem~\ref{thm:diagonal-global-linear-convergence}, followed by the action-value, policy, and unregularized corollaries. To prepare for the proof of Theorem~\ref{thm:diagonal-global-linear-convergence}, we first isolate three technical lemmas.  Lemma~\ref{lem:diagonal-bellman-violation} handles an arbitrary initial critic, Lemma~\ref{lem:diagonal-resolvent-comparison} allows the error to be evaluated under an arbitrary state distribution, and Lemma~\ref{lem:diagonal-weighted-recursion} combines the critic and policy errors into a single potential recursion.  Iterating this recursion gives the value bound in Theorem~\ref{thm:diagonal-global-linear-convergence} under the basic mirror-map condition.  The Bellman representation then gives the action-value bound in Corollary~\ref{cor:diagonal-pointwise-Q}, while the additional strong-convexity condition converts the Bregman term into the policy-distance bound in Corollary~\ref{cor:diagonal-policy-convergence}. Throughout the rest of this section, assume that Assumptions~\ref{ass:h} and \ref{ass:policy-initialization} hold,  $W$ satisfies \eqref{eq:diagonal-relaxation-assumption}, and $Q_\tau^0$ is finite.

The following Bellman violation quantity plays an important role in the analysis:
\begin{equation}
\delta_k:=\frac{\bigl\|W[Q_\tau^k-\calF_\tau^{\pi_k}Q_\tau^k]_+\bigr\|_\infty}{\underline w(1-\gamma)},\qquad k\geq 0.
\label{eq:diagonal-violation-definition}
\end{equation}
It measures how far the current critic is from being a Bellman subsolution for $\pi_k$ that satisfies $Q_\tau^k\leq\calF_\tau^{\pi_k}Q_\tau^k$, corresponding to  $\delta_k=0$.  Lemma~\ref{lem:diagonal-bellman-violation} shows that $\delta_k$ decays linearly and can be used to characterize the difference between $Q_\tau^k$ and $Q_\tau^{\pi_k}$.

\begin{lemma}[Bellman violation and critic comparison]
\label{lem:diagonal-bellman-violation}
For every $k\geq0$,
\begin{equation}
0\leq\delta_{k+1}\leq[1-\underline w(1-\gamma)]\delta_k,
\qquad
Q_\tau^k-\delta_k\mathbf1\leq Q_\tau^{\pi_k}\leq Q_\tau^*.
\label{eq:diagonal-violation-conclusions}
\end{equation}
Moreover, $Q_\tau^k-\delta_k\mathbf1\leq Q_\tau^{\pi_{k+1}}$.
\end{lemma}

In order to handle a general $\mu\in\Delta(\calS)$, we need to introduce an auxiliary distribution that is tied to $\mu$:
\[
(\nu_{\mu,\xi}^*)^\top:=
\left(1-\frac{\gamma}{\xi}\right)
\sum_{t=0}^\infty\left(\frac{\gamma}{\xi}\right)^t
\mu^\top\bigl(P_{\scriptscriptstyle\calS}^{\pi_\tau^*}\bigr)^t,
\]
where $\xi\in(\gamma,1)$ is fixed. Define
\[
\gamma_{\mu,\xi}:=\xi-(\xi-\gamma)
\min_{s:\,\nu_{\mu,\xi}^*(s)>0}
\frac{\mu(s)}{\nu_{\mu,\xi}^*(s)}.
\]
Then $\gamma_{\mu,\xi}\in[\gamma,\xi]$.  The next lemma gives the relations between $\mu$ and $\nu_{\mu,\xi}^*$. The first relation shows exactly how one transition under $\pi_\tau^*$ links $\nu_{\mu,\xi}^*$ back to itself and to $\mu$, and this is the identity that replaces stationarity.  The scalar $\gamma_{\mu,\xi}$ records the loss incurred when comparing their weights, as quantified by the second relation. 

\begin{lemma}[Change of distribution]
\label{lem:diagonal-resolvent-comparison}
For every $s\in\calS$, the preceding definitions satisfy
\begin{equation}
\gamma(\nu_{\mu,\xi}^*)^\top
P_{\scriptscriptstyle\calS}^{\pi_\tau^*}
=\xi(\nu_{\mu,\xi}^*)^\top-(\xi-\gamma)\mu^\top,
\qquad
(\xi-\gamma)\mu(s)
\geq(\xi-\gamma_{\mu,\xi})\nu_{\mu,\xi}^*(s)
\label{eq:diagonal-resolvent-comparison}
\end{equation}
\end{lemma}

Define the potential function
\begin{equation}
\mathcal L_k
:=\sum_{s,a}\frac{\nu_{\mu,\xi}^*(s)\pi_\tau^*(a\mid s)}{w(s,a)}
\bigl(Q_\tau^*(s,a)-Q_\tau^k(s,a)+\delta_k\bigr)
+\max \{ \gamma_{\mu,\xi}(\eta^{-1}+\tau), \eta^{-1}\}
D_{\pi_k}^{\pi_\tau^*}(\nu_{\mu,\xi}^*),
\label{eq:diagonal-arbitrary-mu-potential}
\end{equation}
and, for brevity, set
\begin{equation}
\rho:= 1 - \underline w \cdot \min\!\left\{1- \gamma_{\mu,\xi}, \; \frac{\eta\tau}{1+\eta\tau}\right\}.
\label{eq:diagonal-contraction-factor}
\end{equation}

\begin{lemma}[Potential recursion]
\label{lem:diagonal-weighted-recursion}
For every $k\geq0$,
\begin{equation}
\mathcal L_{k+1}\leq\rho\mathcal L_k
+(1-\gamma)\left(1-\underline w
\sum_{s,a}\frac{\nu_{\mu,\xi}^*(s)\pi_\tau^*(a\mid s)}{w(s,a)}\right)
\delta_k.
\label{eq:diagonal-exact-energy-recursion}
\end{equation}
\end{lemma}

For $k\geq0$, define
\begin{equation}
\mathcal C_k:=\rho^k\mathcal L_0
+(1-\gamma)\left(1-\underline w
\sum_{s,a}\frac{\nu_{\mu,\xi}^*(s)\pi_\tau^*(a\mid s)}{w(s,a)}\right)
\sum_{j=0}^{k-1}\rho^{k-1-j}\delta_j.
\label{eq:diagonal-arbitrary-mu-iterated-constant}
\end{equation}
 Furthermore, for any finite set $\mathcal X$ and distributions $\zeta,\nu\in\Delta(\mathcal X)$ satisfying $\operatorname{supp}(\zeta)\subseteq\operatorname{supp}(\nu)$, define
\begin{equation}
\left\|\frac{\zeta}{\nu}\right\|_\infty
:=
\max_{x\in\operatorname{supp}(\nu)}
\frac{\zeta(x)}{\nu(x)}.
\label{eq:density-ratio-infinity-norm}
\end{equation}
We are now ready to present the main result of this section.
\begin{theorem}[Value convergence under arbitrary evaluation distributions]
\label{thm:diagonal-global-linear-convergence}
Suppose Assumptions~\ref{ass:normalization}, \ref{ass:h}, and \ref{ass:policy-initialization} hold, let $W$ satisfy \eqref{eq:diagonal-relaxation-assumption}, and let $Q_\tau^0$ be finite. Fix $\mu\in\Delta(\calS)$ and $\xi\in(\gamma,1)$.  For every $k\geq0$,
\begin{equation}
0\leq V_\tau^*(\mu)-V_\tau^{\pi_{k+1}}(\mu)
\leq\left\|\frac{\mu}{\nu_{\mu,\xi}^*}\right\|_\infty
\mathcal C_k.
\label{eq:diagonal-arbitrary-mu-value-rate}
\end{equation}
In particular,
\[
0\leq V_\tau^*(\nu^*)-V_\tau^{\pi_{k+1}}(\nu^*)
\leq \mathcal C_k.
\]
Here $\mathcal C_k$ is evaluated with $\mu=\nu^*$, so that $\nu_{\mu,\xi}^*=\nu^*$ and $\gamma_{\mu,\xi}=\gamma$.
\end{theorem}
%Before proceeding, we discuss $\mathcal{C}_k$, the choice of $\xi$, and the effect of the coordinate weights.

\paragraph{Structure and decay of $\mathcal C_k$.}
The bound $\mathcal C_k$ consists of the initial-error term $\rho^k\mathcal L_0$ and a term accumulating the Bellman violations. The coefficient of the latter term measures variation in the coordinate weights under $\nu_{\mu,\xi}^*(s)\pi_\tau^*(a\mid s)$, as seen from the identity
\[
1-\underline w\sum_{s,a}
\frac{\nu_{\mu,\xi}^*(s)\pi_\tau^*(a\mid s)}{w(s,a)}
=\sum_{s,a}\nu_{\mu,\xi}^*(s)\pi_\tau^*(a\mid s)
\left(1-\frac{\underline w}{w(s,a)}\right).
\]
This coefficient is nonnegative and vanishes when all coordinate weights are equal, in which case $\mathcal C_k = \rho^k \mathcal L_0$.
For general coordinate weights, Lemma~\ref{lem:diagonal-bellman-violation} yields
\(
\delta_j\leq[1-\underline w(1-\gamma)]^j\delta_0
\). Since $1-\underline w(1-\gamma)\leq\rho<1$, it follows that
\[
\mathcal C_k=
\begin{cases}
O(\rho^k),
& \mbox{if } \, 1-\underline w(1-\gamma)<\rho,\\
O((k+1)\rho^k),
& \mbox{if } \,  1-\underline w(1-\gamma)=\rho.
\end{cases}
\]
Thus the bound enjoys linear convergence for every finite initial critic.

For $W=I$, the contraction factor reduces to 
\[\rho = \max\{\gamma_{\mu,\xi}, (1+\eta \tau)^{-1}\}.\]
Moreover, the Bellman-violation term vanishes, so $\mathcal C_k = \rho^k \mathcal L_0$.
If, $\mu=\nu^*$, where $\nu^*$ is stationary under $\pi_\tau^*$, then $\gamma_{\mu,\xi}=\gamma$, and the density-ratio factor equals one. The value bound therefore simplifies to 
\[
0\leq V_\tau^*(\nu^*)-V_\tau^{\pi_{k+1}}(\nu^*)
\leq\rho^k\mathcal L_0,
\qquad
\rho=\max\left\{\gamma,(1+\eta\tau)^{-1}\right\}.
\]
For $1+\eta\tau\geq\gamma^{-1}$, this factor equals $\gamma$,
matching the rate in the strongly convex PMD theorem of
\citet{Lan_2021}.
For smaller stepsizes, the factor is governed by
$(1+\eta\tau)^{-1}$.
Thus the bound covers every $\eta>0$.
For comparison, \citet{Zhan_Cen_Huang_Chen_Lee_Chi_2021}
obtain the linear-convergence factor
$1-\eta\tau(1-\gamma)/(1+\eta\tau) \ge \max \{\gamma, (1+\eta\tau)^{-1}\}$.

\paragraph{Choice of $\xi$.} First note that the density-ratio factor in Theorem~\ref{thm:diagonal-global-linear-convergence} is bounded even when $\mu$ lacks full support. Indeed, by the definition of $\nu_{\mu,\xi}^*$, we have $\nu_{\mu,\xi}^*(s)\geq(1-\gamma/\xi)\mu(s)$ for every state, and hence
\begin{equation}
\left\|\frac{\mu}{\nu_{\mu,\xi}^*}\right\|_\infty
\leq\frac{\xi}{\xi-\gamma}<\infty.
\label{eq:diagonal-density-ratio-bounds}
\end{equation}
Moreover, the bounds display a finite-time tradeoff. The inequality $\gamma_{\mu,\xi}\leq\xi$ gives 
\[\rho \le 1 -  \underline w \cdot \min \bigg\{ 1- \xi, \; \frac{\eta\tau}{1+\eta\tau}\bigg\}.\] Choosing $\xi$ closer to $\gamma$ decreases the critic term in this upper bound. On the other hand, the universal density-ratio bound $\xi/(\xi-\gamma)$ decreases as $\xi$ increases and diverges as $\xi\downarrow\gamma$.  The midpoint choice $\xi=(1+\gamma)/2$ gives the simple bounds
\[
\gamma_{\mu,\xi}\leq\frac{1+\gamma}{2},
\qquad
\left\|\frac{\mu}{\nu_{\mu,\xi}^*}\right\|_\infty
\leq\frac{1+\gamma}{1-\gamma}.
\]

\paragraph{Coordinate weighting can improve convergence.}
Although \(W=I\) yields the smallest contraction factor in the theoretical bound, nonuniform coordinate weights can nevertheless lead to faster convergence on certain MDPs. The following example illustrates this. Consider a two-state, two-action MDP with $\calS=\{s_1,s_2\}$ and $\calA=\{a_1,a_2\}$.  Both actions move deterministically to the other state, and, for $i=1,2$, let
\[
r(s_i,a_1)=\tfrac12,\qquad r(s_i,a_2)=0,\qquad
\gamma=\tfrac12,\qquad \tau=\eta=1,\qquad
h(p)=\tfrac12\|p\|_2^2.
\]
For a policy $\pi$, write $p(s_i)=\pi(a_1\mid s_i)$. The regularized one-step reward at $s_i$ is \[\frac{p(s_i)}2-\frac12\bigl[p(s_i)^2+(1-p(s_i))^2\bigr] =\frac1{16}-\left(p(s_i)-\frac34\right)^2.\] Since transitions do not depend on the action, the optimal policy is $\pi_\tau^*(\cdot\mid s_i)=(3/4,1/4)$ at both states. For evaluation distribution $\mu=(1/2,1/2)$, the value gap is $V_\tau^*(\mu)-V_\tau^\pi(\mu) =\sum_{i=1}^2\left(p(s_i)-\frac34\right)^2$.  Let $\pi_k$ and $\pi_k'$ denote the policy sequences generated by Algorithm~\ref{alg:exact-diagonal-td-pmd} with 
\begin{equation}
\label{eq:coordinate-weight-examples}
W=\operatorname{diag}(3/4,3/4,1,1),\qquad W'=I,
\end{equation} respectively.
Write $p_k(s_i)=\pi_k(a_1\mid s_i)$ and $p_k'(s_i)=\pi_k'(a_1\mid s_i)$ for their respective probabilities of choosing $a_1$ at $s_i$.  Both runs use the same initialization:
\begin{equation}
\label{eq:common-initialization-examples}
Q_\tau^0=
\begin{pmatrix}3/4&0\\1&0\end{pmatrix},\qquad
p_0(s_i)=p_0'(s_i)=\tfrac12,\quad i=1,2.
\end{equation}
Both actions at each state lead to the same next state, so their Bellman targets differ only in the immediate reward. At each state, the two actions have equal update weights in each run, so subtracting the critic updates gives
\begin{equation}
\label{eq:critic-update-example}
Q_\tau^{k+1}(s_i,a_1)-Q_\tau^{k+1}(s_i,a_2)
=\bigl(1-w(s_i,a_1)\bigr)
\bigl[Q_\tau^k(s_i,a_1)-Q_\tau^k(s_i,a_2)\bigr]
+\tfrac12w(s_i,a_1).
\end{equation}
The policy update along these trajectories is
\[
p_{k+1}(s_i)=\tfrac12p_k(s_i)+\tfrac14
+\tfrac14\bigl[Q_\tau^k(s_i,a_1)-Q_\tau^k(s_i,a_2)\bigr].
\]
Applying \eqref{eq:critic-update-example} with \eqref{eq:coordinate-weight-examples} and \eqref{eq:common-initialization-examples}, we obtain, for every $k\geq1$,
\begin{equation}
\label{eq:policy-probability-examples}
p_k(s_1)=\tfrac34-\tfrac14\,4^{-k},\qquad
p_k'(s_1)=\tfrac34-\tfrac18\,2^{-k},\qquad
p_k(s_2)=p_k'(s_2)=\tfrac34.
\end{equation}
Substituting \eqref{eq:policy-probability-examples} into the value gap gives, for every $k\geq1$,
\[
V_\tau^*(\mu)-V_\tau^{\pi_k}(\mu)
=\tfrac1{16}\,16^{-k},\qquad
V_\tau^*(\mu)-V_\tau^{\pi_k'}(\mu)
=\tfrac1{64}\,4^{-k}.
\]
Thus $W$ gives a strictly smaller value gap than $W'$ for every $k\geq2$ and improves its exact exponential factor from $1/4$ to $1/16$.

\vspace{1em}

Fix $(s,a)$ and set $\mu=P(\cdot\mid s,a)$. By the Bellman equations, one has
\[
Q_\tau^*(s,a)-Q_\tau^{\pi_{k+1}}(s,a)
=\gamma\bigl[V_\tau^*(\mu)-V_\tau^{\pi_{k+1}}(\mu)\bigr].
\]
Applying Theorem~\ref{thm:diagonal-global-linear-convergence} establishes the convergence of the action values.

\begin{corollary}[Action-value convergence]
\label{cor:diagonal-pointwise-Q}
Suppose Assumptions~\ref{ass:normalization}, \ref{ass:h}, and \ref{ass:policy-initialization} hold, let $W$ satisfy \eqref{eq:diagonal-relaxation-assumption}, and let $Q_\tau^0$ be finite. Fix $(s,a)\in\calS\times\calA$, $\xi\in(\gamma,1)$, and set $\mu=P(\cdot\mid s,a)$.  With $\nu_{\mu,\xi}^*$, $\gamma_{\mu,\xi}$, and $\mathcal C_k$ evaluated at this choice of $\mu$, we obtain
\begin{equation}
0\leq Q_\tau^*(s,a)-Q_\tau^{\pi_{k+1}}(s,a)
\leq\gamma\left\|\frac{\mu}{\nu_{\mu,\xi}^*}\right\|_\infty
\mathcal C_k.
\label{eq:diagonal-arbitrary-mu-Q-rate}
\end{equation}
\end{corollary}

If we further assume the strong convexity of $h$ (Assumption~\ref{ass:strong-convexity}), then by the definition of $\mathcal L_k$,
\[
\frac{\lambda}{2}\E_{s\sim\nu_{\mu,\xi}^*}\!\left[
\|\pi_k(\cdot\mid s)-\pi_\tau^*(\cdot\mid s)\|_1^2\right]
\leq D_{\pi_k}^{\pi_\tau^*}(\nu_{\mu,\xi}^*)
\leq\frac{\eta}{\max \{ \gamma_{\mu,\xi}(1+\eta\tau), 1\}}\mathcal L_k.
\]
Using $\mu(s)\leq\|\mu/\nu_{\mu,\xi}^*\|_\infty\nu_{\mu,\xi}^*(s)$ for every state transfers this estimate to the expectation under $\mu$. The proof of Theorem~\ref{thm:diagonal-global-linear-convergence} establishes the bound \(\mathcal L_k\leq\mathcal C_k\) in \eqref{eq:diagonal-iterated-potential-bound}, which immediately yields the policy convergence of Algorithm~\ref{alg:exact-diagonal-td-pmd}.

\begin{corollary}[Policy convergence under strong convexity]
\label{cor:diagonal-policy-convergence}
Fix $\mu\in\Delta(\calS)$ and $\xi\in(\gamma,1)$.  Suppose in addition that Assumption~\ref{ass:strong-convexity} holds.  Then, for every $k\geq0$,
\begin{equation}
\E_{s\sim\mu}\!\left[
\|\pi_k(\cdot\mid s)-\pi_\tau^*(\cdot\mid s)\|_1^2\right]
\leq\frac{2\eta}{\lambda\max \{ \gamma_{\mu,\xi}(1+\eta\tau), 1 \}}
\left\|\frac{\mu}{\nu_{\mu,\xi}^*}\right\|_\infty
\mathcal C_k.
\label{eq:diagonal-arbitrary-mu-policy-rate}
\end{equation}
\end{corollary}

The argument used to prove Theorem~\ref{thm:diagonal-global-linear-convergence} also applies when $\tau=0$, and yields two types of convergence: a constant stepsize gives an averaged sublinear rate, whereas geometrically increasing stepsizes give a last-iterate linear rate. This agrees with the behavior established for unregularized one-step TD–PMD by \citet{liu2025tdpmd}. For clarity, the proof of the following result is also included in Appendix~\ref{app:exact-diagonal-proof}.

\begin{corollary}[Unregularized exact TD--PMD]
\label{cor:unregularized-exact-td-pmd}
Consider the variable-stepsize variant of Algorithm~\ref{alg:exact-diagonal-td-pmd} with $\tau=0$ and $W=I$, where the policy stepsize at iteration \(k\) is $\eta_k >0$, and omit the subscript $\tau$.  Fix an unregularized optimal policy $\pi^*$ and a stationary distribution $\nu^*$ of $P_{\scriptscriptstyle\calS}^{\pi^*}$, and let $\delta_k$ be defined by \eqref{eq:diagonal-violation-definition} with $\tau=0$ and $W=I$. If $\eta_k\equiv\eta>0$, then, for every $K\geq1$,
\begin{equation}
\frac{1}{K}\sum_{k=0}^{K-1}
\left[V^*(\nu^*)-V^{\pi_{k+1}}(\nu^*)\right]
\leq\frac{
\sum_{s,a}\nu^*(s)\pi^*(a\mid s)
\bigl(Q^*(s,a)-Q^0(s,a)+\delta_0\bigr)
+\eta^{-1}D_{\pi_0}^{\pi^*}(\nu^*)
}{(1-\gamma)K},
\label{eq:unregularized-constant-step-rate}
\end{equation}
where $V^\pi$ and $V^*$ denote the unregularized value function of $\pi$ and optimal value function, respectively. If instead the positive stepsizes satisfy $\eta_{k+1}\geq\eta_k/\gamma$ for every $k\geq0$, then the last iterate satisfies
\begin{equation}
V^*(\nu^*)-V^{\pi_{k+1}}(\nu^*)
\leq\gamma^k
\left\{
\sum_{s,a}\nu^*(s)\pi^*(a\mid s)
\bigl(Q^*(s,a)-Q^0(s,a)+\delta_0\bigr)
+\eta_0^{-1}D_{\pi_0}^{\pi^*}(\nu^*)
\right\}.
\label{eq:unregularized-increasing-step-rate}
\end{equation}
\end{corollary}

%%%%%%%%%%%%%%%
\section{Regularized TD--PMD under Off-Policy Markov Data}
\label{sec:off-policy-markov-data}

In this section, we consider the sample complexity of regularized TD--PMD under off-policy Markov data. The method is summarized in Algorithm~\ref{alg:Expected-Regularized-TD-PMD} and maintains a target policy $\pi_k$ and a critic $Q_\tau^k$, while all data are generated by a fixed behavior policy $\pi_b$. At iteration $k$, after updating the target policy from $\pi_k$ to $\pi_{k+1}$, the algorithm collects a batch of $B_k$ consecutive transitions
\[
\tau_k
=
\{(s_t^k,a_t^k,r_t^k,s_{t+1}^k)\}_{t=0}^{B_k-1}
\]
along the behavior trajectory, where $a_t^k \sim \pi_b(\cdot \mid s_t^k)$ and $s_{t+1}^k \sim P(\cdot \mid s_t^k,a_t^k)$. The trajectory is not reset between iterations: the terminal state of the current batch becomes the initial state of the next one, $s_0^{k+1}=s_{B_k}^k$. For each sampled transition, we form a coordinate-wise TD increment whose target is
\[
r_t^k
+
\gamma \mathbb{E}_{a\sim\pi_{k+1}(\cdot\mid s_{t+1}^k)}
\bigl[Q_\tau^k(s_{t+1}^k,a)\bigr]
-
\tau\gamma h^{\pi_{k+1}}(s_{t+1}^k).
\]
The resulting TD increments are aggregated using deterministic weights $\{c_t^k\}_{t=0}^{B_k-1}$, and their weighted average is used to perform one critic update with stepsize $\alpha_k$.

\begin{algorithm}[ht!]
    \small
    \caption{Finite-Batch Off-Policy Expected TD--PMD}
    \label{alg:Expected-Regularized-TD-PMD}
\begin{algorithmic}
    \STATE {\bfseries Input:} Iterations $K$, initial action-value vector $Q_\tau^0=0$,
    initial policy $\pi_0$, critic stepsizes $\{\alpha_k\}$, constant policy
    stepsize $\eta>0$, initial state $s_0$, batch sizes $\{B_k\}$, and
    averaging weights $\{c_t^k\}$.
    \STATE Set $s_0^0 = s_0$.
    \FOR{$k=0,1, \dots, K-1$}
    \STATE \textbf{(Policy update)} Update the target policy by
    \begin{align*}
        \forall\, s\in\calS: \quad \pi_{k+1}(\cdot|s) = \underset{p\in\Delta(\calA)}{\arg\max} \; \ls\{ \ls\langle p ,\, Q_\tau^k(s,\cdot) \rs\rangle - \tau h(p) - \frac{1}{\eta} D^p_{\pi_k}(s) \rs\}.
    \end{align*}
    \STATE \textbf{(Sampling)} Starting from $s_0^k$, collect the consecutive
    batch $\tau_k$ under the behavior policy $\pi_b$,
    \begin{align*}
        \tau_k = \{ (s_t^k, a_t^k, r^k_t, s_{t+1}^k) \}_{t=0}^{B_k-1}, \quad \mbox{where} \; r^k_t = r(s^k_t, a^k_t), \;\; a^k_t \sim \pi_b(\cdot|s_t^k), \;\; s^k_{t+1} \sim P(\cdot| s_t^k, a_t^k),
    \end{align*}
    and set $s_0^{k+1} = s_{B_k}^k$.
    \STATE \textbf{(Critic update)} Construct the expected TD error,
    \begin{align*}
        \delta_t^k(s,a) &:= \mathds{1}[(s^k_t,a^k_t)=(s,a)]\cdot\ls[r_t^k + \gamma\E_{a\sim\pi_{k+1}(\cdot|s_{t+1}^k)} \ls[Q_\tau^k(s_{t+1}^k, a)\rs] - \tau \gamma h^{\pi_{k+1}}(s^k_{t+1}) - Q_\tau^k(s^k_t, a^k_t)\rs], \\
        \bar\delta_k(s,a) &:= {\sum_{t=0}^{B_k-1} c_t^k \cdot \delta_t^k(s,a)}.
    \end{align*}
    \STATE Update the critic by
    \begin{align*}
        Q_\tau^{k+1}(s,a) = Q_\tau^k(s,a) + \alpha_k \cdot \bar\delta_k(s,a).
    \end{align*}
    \ENDFOR
    \STATE Independently sample $\widehat K\in\{0,\ldots,K-1\}$ according to
    \begin{equation}
    \mathbb P(\widehat K=k)
    =
    \frac{(1-\rho)\rho^{K-k-1}}{1-\rho^K},
    \qquad k=0,\ldots,K-1,
    \quad \rho=(1+\eta\tau)^{-1}.
    \label{eq:exponentially-weighted-output-index}
    \end{equation}
    \STATE Output $\pi_{\widehat K}$.
\end{algorithmic}
\end{algorithm}

\begin{assumption}[Behavior-policy coverage and ergodicity]
\label{ass:off-policy-exploration}
The behavior policy has full support:
$\pi_b(a\mid s)>0$ for every $(s,a)\in\calS\times\calA$. The transition matrix $P_{\scriptscriptstyle\calS}^{\pi_b}$ is irreducible and aperiodic.
\end{assumption}

Under Assumption~\ref{ass:off-policy-exploration}, the behavior
chain has a unique stationary distribution
$\nu^{\pi_b}\in\Delta(\calS)$ with strictly positive entries
\citep[Corollary~1.17 and Proposition~1.19]{levin2017markov},
satisfying
\begin{equation}
(\nu^{\pi_b})^\top P_{\scriptscriptstyle\calS}^{\pi_b}
=(\nu^{\pi_b})^\top,\qquad \nu^{\pi_b}(s)>0\quad(s\in\calS).
\label{eq:behavior-stationarity}
\end{equation}
Moreover, by \citet[Theorem~4.9]{levin2017markov},
there exist constants $m_b>0$ and $\kappa_b\in(0,1)$ such that
\begin{equation}
d_{\mathrm{TV}}\!\left(
\bigl(P_{\scriptscriptstyle\calS}^{\pi_b}\bigr)^t(s,\cdot),
\nu^{\pi_b}
\right)\leq m_b\kappa_b^t,\qquad s\in\calS,\quad t\geq0.
\label{eq:behavior-exponential-mixing}
\end{equation}

We also assume the following standard conditions on the initialization, critic stepsizes, batch lengths, and within-batch averaging weights. Note that the restriction $Q_\tau^0=0$ in the assumption is made only to simplify the displayed constants.

\begin{assumption}[Algorithmic conditions]
\label{ass:markov-data-algorithm}
The critic initialization and parameters satisfy
$Q_\tau^0=0$, $0<\alpha_k=\alpha\leq1$, and $B_k\geq1$.  The weights
$\{c_t^k\}_{t=0}^{B_k-1}$ are deterministic and satisfy
$c_t^k\geq0$ and $\sum_{t=0}^{B_k-1}c_t^k=1$.
\end{assumption}

To proceed, define the following quantities based on the stationary distribution $\nu^{\pi_b}$:
\begin{align}
\sigma^{\pi_b}(s,a)&:=\nu^{\pi_b}(s)\pi_b(a\mid s),\qquad
\Sigma_b:=\operatorname{diag}\!\left(\sigma^{\pi_b}(s,a):s,a\right),
\nonumber\\
\widetilde\sigma_b&:=\min_{s,a}\sigma^{\pi_b}(s,a)>0,\qquad
\overline\sigma_b:=\max_{s,a}\sigma^{\pi_b}(s,a).
\label{eq:behavior-distribution-constants}
\end{align}
If each sampled TD increment is replaced by its expectation under the stationary behavior distribution, it is not hard to see that the critic update reduces to the coordinate-wise update in Algorithm~\ref{alg:exact-diagonal-td-pmd} with $W=\alpha\Sigma_b$, where the coordinate weights are determined by the stationary state--action visitation probabilities.  Thus Theorem~\ref{thm:diagonal-global-linear-convergence} immediately yields a noise-free value-convergence benchmark. 
 \begin{corollary}[Noise-free behavior-weighted value convergence]
\label{cor:behavior-weighted-exact}
Suppose the assumptions of Theorem~\ref{thm:diagonal-global-linear-convergence} and Assumption~\ref{ass:off-policy-exploration}
hold.   Fix $\xi\in(\gamma,1)$, and let $W=\alpha\Sigma_b$ with
$0<\alpha\overline\sigma_b\leq1$.  Consider the stationary distribution $\nu^*$ of $P_{\scriptscriptstyle\calS}^{\pi_\tau^*}$ for simplicity. Since $\nu_{\nu^*,\xi}^*=\nu^*$ and $\gamma_{\nu^*,\xi}=\gamma$, a direct application of Theorem~\ref{thm:diagonal-global-linear-convergence} gives
\[
0\leq V_\tau^*(\nu^*)-V_\tau^{\pi_{k+1}}(\nu^*)
\leq \mathcal C_k,
\]
where $\mathcal C_k$ in \eqref{eq:diagonal-arbitrary-mu-iterated-constant} is evaluated with $\mu=\nu^*$ and $W = \alpha \Sigma_b$, so that $\underline w=\alpha\widetilde\sigma_b$ and $\overline w=\alpha\overline\sigma_b$. \end{corollary}
Next, we are going to show that the stochastic variant (i.e., Algorithm~\ref{alg:Expected-Regularized-TD-PMD}) attains an expected
value gap of at most $\epsilon$ after
\(
\widetilde{\mathcal{O}}
\left(
\frac{1}{(1-\gamma)^5\widetilde{\sigma}_b\epsilon}
\right)
\)
observed transitions. The analysis requires bounds on the additional errors caused by finite batches, temporal dependence, and nonstationary batch initialization. As in the previous section, we first present the technical results used in the proof of the main result. The first result shows that the critic iterates are uniformly bounded. 

\begin{lemma}[Bounded critic iterates]
\label{lem:bounded-stochastic-critic-iterates}
Under Assumption~\ref{ass:markov-data-algorithm}, for every $k\geq0$,
\[
|Q_\tau^k(s,a)|\leq
\frac{1+\tau\gamma H_h}{1-\gamma},
\qquad (s,a)\in\calS\times\calA.
\]
\end{lemma}

Because the target policy changes from one batch to the next, we also need to control the variation of the policy iterates and the  exact action-value targets along the PMD trajectory. The following lemma derives this trajectory-local regularity from the Bregman geometry of the iterates, without imposing a global Lipschitz condition on the regularizer or its gradient. 

\begin{lemma}[Trajectory regularity along the PMD iterates] 
\label{lem:trajectory-regularity}
Define  
\begin{equation}
L_h
:=
H_h+\frac12\max\left\{
\max_{s\in\calS,\,a\in\calA}
D_h\!\left(e_a\,\middle\|\,\pi_0(\cdot\mid s)\right),
\frac{2(1+\tau H_h)}{\tau(1-\gamma)}
\right\}.
\label{eq:trajectory-regularity-constants}
\end{equation}
For every $k\geq0$ and $s\in\calS$,
\begin{equation}
\label{eq:h-trajectory-lipschitz-proved}
\left|
h^{\pi_{k+1}}(s)-h^{\pi_k}(s)
\right|
\leq
L_h
\left\|
\pi_{k+1}(\cdot\mid s)-\pi_k(\cdot\mid s)
\right\|_1.
\end{equation}
Moreover,
\begin{equation}
\|\pi_{k+1}-\pi_k\|_{1,\infty}
\leq
\frac{\eta}{\lambda}
\left(
\frac{1+\tau\gamma H_h}{1-\gamma}+\tau L_h
\right),
\label{eq:uniform-policy-increment-value-section}
\end{equation}
where
$\|\pi-\pi'\|_{1,\infty}
:=\max_{s\in\calS}\|\pi(\cdot\mid s)-\pi'(\cdot\mid s)\|_1$, and
\begin{equation}
\label{eq:q-trajectory-lipschitz}
\left\|
Q_\tau^{\pi_{k+1}}-Q_\tau^{\pi_k}
\right\|_\infty
\leq
\frac{\gamma}{1-\gamma}
\left(
\frac{1+\tau\gamma H_h}{1-\gamma}+\tau L_h
\right)
\|\pi_{k+1}-\pi_k\|_{1,\infty}.
\end{equation}
\end{lemma}

We also need to separate the deterministic benchmark from the error caused by Markov sampling.  Identify $\delta_t^k$ and $\bar\delta_k$ with vectors in $\mathbb R^{|\calS||\calA|}$ and define the centered errors
\begin{equation}
\omega_t^k:=\delta_t^k-\Sigma_b
\bigl(\calF_\tau^{\pi_{k+1}}Q_\tau^k-Q_\tau^k\bigr),
\qquad \bar\omega_k:=\sum_{t=0}^{B_k-1}c_t^k\omega_t^k.
\label{eq:stochastic-critic-error-definition}
\end{equation}
The critic recursion is therefore given by
\begin{equation}
Q_\tau^{k+1}=Q_\tau^k+\alpha_k\Sigma_b
\bigl(\calF_\tau^{\pi_{k+1}}Q_\tau^k-Q_\tau^k\bigr)
+\alpha_k\bar\omega_k.
\label{eq:expected-TD-PMD-critic-update}
\end{equation}
Let
\begin{equation}
\mathcal H_k
:=
\sigma\!\left(
s_0^0,\,
\left\{
(s_t^\ell,a_t^\ell,r_t^\ell,s_{t+1}^\ell):
0\leq\ell<k,\ 0\leq t<B_\ell
\right\}
\right)
\label{eq:critic-history}
\end{equation}
be the information available before the $k$-th batch is sampled.  In particular, $s_0^k$, $Q_\tau^k$, and $\pi_{k+1}$ are $\mathcal H_k$-measurable. The conditional sampling-bias can be bounded as follows.
\begin{lemma}[Conditional bias of the stochastic critic error]
\label{lem:stochastic-error-bound}
Under Assumption~\ref{ass:markov-data-algorithm}, for every $k\geq0$ and
$(s,a)\in\calS\times\calA$,
\begin{equation}
\left|
\mathbb E\!\left[\bar\omega_k(s,a)\mid\mathcal H_k\right]
\right|
\leq
\frac{2m_b(1+\tau\gamma H_h)}{1-\gamma}
\sum_{t=0}^{B_k-1}c_t^k\kappa_b^t.
\label{eq:stochastic-error-bound}
\end{equation}
Equivalently, the same bound holds for
$\|\mathbb E[\bar\omega_k\mid\mathcal H_k]\|_\infty$, and  consequently it
also holds for $|\mathbb E[\bar\omega_k(s,a)]|$.
\end{lemma}

The factor $\kappa_b^t$ measures how much dependence on the initial state remains after $t$ steps of the behavior chain, while $c_t^k$ records how much weight the critic assigns to that sample.  Their weighted sum therefore quantifies the conditional bias of one batch.  The bias is reduced when the weighting scheme places sufficient mass on later, better-mixed samples. The exponential weights used in Theorem~\ref{thm:off-policy-markov-value-gap} make this dependence decay geometrically with the batch length.

\begin{theorem}[Expected value-gap bound]
\label{thm:off-policy-markov-value-gap}
Suppose Assumptions~\ref{ass:normalization}, \ref{ass:h},
\ref{ass:strong-convexity}, \ref{ass:policy-initialization}, \ref{ass:off-policy-exploration},
and \ref{ass:markov-data-algorithm} hold.  Fix
$\vartheta\in[0,\kappa_b)$ and run
Algorithm~\ref{alg:Expected-Regularized-TD-PMD} with
$\alpha_k=\alpha$, $B_k=B$, and
\begin{equation}
c_t^k
:=
\frac{\vartheta^{B-t-1}}
{\sum_{\ell=0}^{B-1}\vartheta^\ell},
\qquad t=0,\ldots,B-1,
\label{eq:exponential-batch-weights-value-section}
\end{equation}
where $0^0:=1$.  Let $\rho:=(1+\eta\tau)^{-1}\in(0,1)$ and assume
$0<\eta\leq
\frac{\alpha(1-\gamma)\widetilde{\sigma}_b}
{\tau\left[2-\alpha(1-\gamma)\widetilde{\sigma}_b\right]}$.
Then, for every integer
$K\geq\left\lceil\log(2)/\log(1/\rho)\right\rceil$, the output
\eqref{eq:exponentially-weighted-output-index} satisfies
\begin{equation}
\mathbb E\!\left[V_\tau^*(\mu)-V_\tau^{\pi_{\widehat K}}(\mu)\right]
\leq
\frac{C_1}{\rho^{-K}-1}
+C_2\eta
+C_3\frac{\kappa_b^B}{\kappa_b-\vartheta}.
\label{eq:explicit-three-term-value-gap}
\end{equation}
where
\(C_1:=\frac{\tau D_{\pi_0}^{\pi_\tau^*}(d_\mu^*)}{1-\gamma}\)
\({}+\frac{4\eta\tau(1+\tau\gamma H_h)}
{\alpha(1-\gamma)^3\widetilde{\sigma}_b}\),
\(C_2:=\frac{4\tau(1+\tau H_h)}{(1-\gamma)^2}\)
\({}+\frac{2(1+\tau\gamma H_h)^2}{\lambda(1-\gamma)^4}\)
\({}+\frac{2(1+\tau\gamma H_h)
\bigl[1+\tau\gamma H_h+\tau(1-\gamma)L_h\bigr]}
{\lambda\alpha(1-\gamma)^4\widetilde{\sigma}_b}\)
\({}+\frac{2\gamma
\bigl[1+\tau\gamma H_h+\tau(1-\gamma)L_h\bigr]
\bigl[9(1+\tau\gamma H_h)+\tau(1-\gamma)L_h\bigr]}
{\lambda\alpha(1-\gamma)^5\widetilde{\sigma}_b}\),
and
\(C_3:=\frac{4m_b(1+\tau\gamma H_h)}
{(1-\gamma)^3\widetilde{\sigma}_b}\).
\end{theorem}

\begin{corollary}[$\epsilon$-accuracy and sample complexity]
\label{cor:off-policy-markov-value-sample-complexity}
Assume that the conditions of Theorem~\ref{thm:off-policy-markov-value-gap} hold.  For any
$0<\epsilon\leq1$, choose
\begin{equation}
\eta_\epsilon
:=
\frac{\epsilon\lambda\alpha(1-\gamma)^5\widetilde{\sigma}_b}
{12\bigl[1+\lambda\tau(1+\tau H_h)\bigr]
\bigl[1+\tau\gamma H_h+\tau(1-\gamma)L_h\bigr]
\bigl[9(1+\tau\gamma H_h)+\tau(1-\gamma)L_h\bigr]}.
\label{eq:epsilon-policy-stepsize}
\end{equation}
Set
\begin{equation}
K_\epsilon
:=
\left\lceil
\frac{2}{\eta_\epsilon\tau}
\log\!\left(
1+\max\!\left\{
1,\frac{3}{\epsilon}
\left[
\frac{\tau D_{\pi_0}^{\pi_\tau^*}(d_\mu^*)}{1-\gamma}
+
\frac{4(1+\tau\gamma H_h)}{(1-\gamma)^2}
\right]\right\}\right)
\right\rceil
\label{eq:epsilon-outer-iterations}
\end{equation}
and
\begin{equation}
B_\epsilon
:=
\max\left\{1,
\left\lceil
\frac{
\log\!\left(
\max\!\left\{
1,\,
\frac{12m_b(1+\tau\gamma H_h)}
{\epsilon(1-\gamma)^3\widetilde{\sigma}_b
(\kappa_b-\vartheta)}
\right\}
\right)
}{
\log(1/\kappa_b)
}
\right\rceil\right\}.
\label{eq:epsilon-batch-length}
\end{equation}
With $\eta=\eta_\epsilon$, $K=K_\epsilon$, and $B=B_\epsilon$, the total number of observed transitions satisfies
$N_\epsilon:=K_\epsilon B_\epsilon
=\widetilde{\mathcal O}\bigl(
[(1-\gamma)^5\widetilde\sigma_b\epsilon]^{-1}\bigr)$, where $\alpha,\tau,\lambda,H_h$, $\max_{s,a}D_h(e_a\,\|\,\pi_0(\cdot\mid s))$, $m_b,\kappa_b$, and $\vartheta$ are treated as fixed constants. These transitions suffice to ensure
\begin{equation}
\mathbb E\!\left[
V_\tau^*(\mu)-V_\tau^{\pi_{\widehat K}}(\mu)
\right]\leq\epsilon.
\label{eq:epsilon-value-guarantee}
\end{equation}
\end{corollary}

The terms $C_1/(\rho^{-K}-1)$, $C_2\eta$, and $C_3\kappa_b^B/(\kappa_b-\vartheta)$ in \eqref{eq:explicit-three-term-value-gap} are, respectively, the decaying optimization error, the constant-stepsize tracking bias, and the finite-batch Markov mixing bias.  Increasing $K$ reduces the first term, while decreasing $\eta$ reduces the tracking bias but slows the contraction $\rho=(1+\eta\tau)^{-1}$.  Increasing $B$ reduces the mixing bias exponentially as $\kappa_b^B$.  The choice $\vartheta=0$ uses only the last sample in each batch; every fixed $\vartheta\in[0,\kappa_b)$ retains this exponential decay in $B$, with the additional factor $(\kappa_b-\vartheta)^{-1}$.  The critic stepsize $\alpha$ and the coverage $\widetilde\sigma_b$ determine how quickly the behavior-weighted critic tracks its target.  Corollary~\ref{cor:off-policy-markov-value-sample-complexity} balances these effects by taking $\eta=\Theta(\epsilon)$, $K=\widetilde{\mathcal O}(1/\epsilon)$, and $B=\mathcal O(\log(1/\epsilon))$ while the remaining problem parameters are fixed. Under near-uniform behavior coverage, $\widetilde\sigma_b=\Theta((|\calS||\calA|)^{-1})$. Thus, Corollary~\ref{cor:off-policy-markov-value-sample-complexity} yields the sample complexity
\(
\widetilde{\mathcal O}\!\left(
\frac{|\calS||\calA|}{\epsilon(1-\gamma)^5}
\right)
\).
This matches the leading-order dependence on $|\calS|$, $|\calA|$, $1-\gamma$, and $\epsilon$ of strongly convex SVMD~\citep{jia2026value}, while our result is established using only a single off-policy Markov trajectory, without trajectory resets or generative-model access.
\section{Numerical Experiments}
\label{sec:numerical-experiments}

We use a randomly generated MDP with $|\calS|=50$ and $|\calA|=10$ to illustrate the behavior of exact coordinate-wise TD--PMD and its finite-batch stochastic counterpart under off-policy Markov data. Each reward $r(s,a)$ is drawn independently from $\operatorname{Unif}[0,1]$. For each state--action pair \((s,a)\), we generate \(P(\cdot\mid s,a)\) by drawing i.i.d. samples from \(\operatorname{Unif}[0,1]\) across the next states, followed by normalization. Both the initial policy $\pi_0$ and the evaluation distribution $\mu$ are uniform. The behavior policy is generated independently as $\pi_b(a\mid s)\propto u_{s,a}$, where $u_{s,a}\sim\operatorname{Unif}[0.5,1.5]$. We consider two standard regularizers:
\begin{equation}
h_{\mathrm{NPG}}(p)=\sum_{a\in\calA}p(a)\log p(a),
\qquad
h_{\mathrm{PQA}}(p)=\frac{1}{2}\|p\|_2^2,
\label{eq:numerical-regularizers}
\end{equation}

For each regularizer \(h\), the optimal policy \(\pi_{\tau,h}^*\) and value function \(V_{\tau,h}^*\) are computed via regularized value iteration with the termination condition \(\|V^{k+1}-V^k\|_\infty\le 10^{-13}\). We then measure the value gap and the mean squared \(\ell_1\) policy error of exact coordinate-wise TD--PMD relative to the corresponding optimum:
\begin{equation}
\mathcal E_V(\pi)
:=V_{\tau,h}^*(\mu)-V_{\tau,h}^{\pi}(\mu),
\qquad
\mathcal E_\pi(\pi)
:=\E_{s\sim\mu}\!\left[
\|\pi(\cdot\mid s)-\pi_{\tau,h}^*(\cdot\mid s)\|_1^2
\right].
\label{eq:numerical-error-metrics}
\end{equation}

\subsection{Exact coordinate-wise update experiments}

The exact experiment illustrates convergence from two critic initializations under the coordinate-wise updates. We use $\gamma=0.95$, $\tau=0.1$, and $\eta=0.5$, and run each experiment for $1000$ iterations with fixed weights $W=\alpha\Sigma_b$ and $\alpha=250$. For the generated environment and behavior policy, $\widetilde\sigma_b=9.565\times10^{-4}$ and $\overline\sigma_b=3.436\times10^{-3}$, giving $w(s,a)=\alpha\sigma^{\pi_b}(s,a)\in[0.239,0.859]$.

The solid curves use $Q_\tau^0=0$, which satisfies the one-sided Bellman condition $\calF_\tau^{\pi_0}Q_\tau^0-Q_\tau^0\geq0$ in this instance. For the dashed curves, a bounded random critic is shifted by the same constant across all state--action pairs  so that $\calF_\tau^{\pi_0}Q_\tau^0-Q_\tau^0\leq-\mathbf1$.  The left panel reports $\mathcal E_V(\pi_k)$, and the right panel reports $\mathcal E_\pi(\pi_k)$.

\begin{figure}[ht!]
\centering
\includegraphics[width=0.98\linewidth]{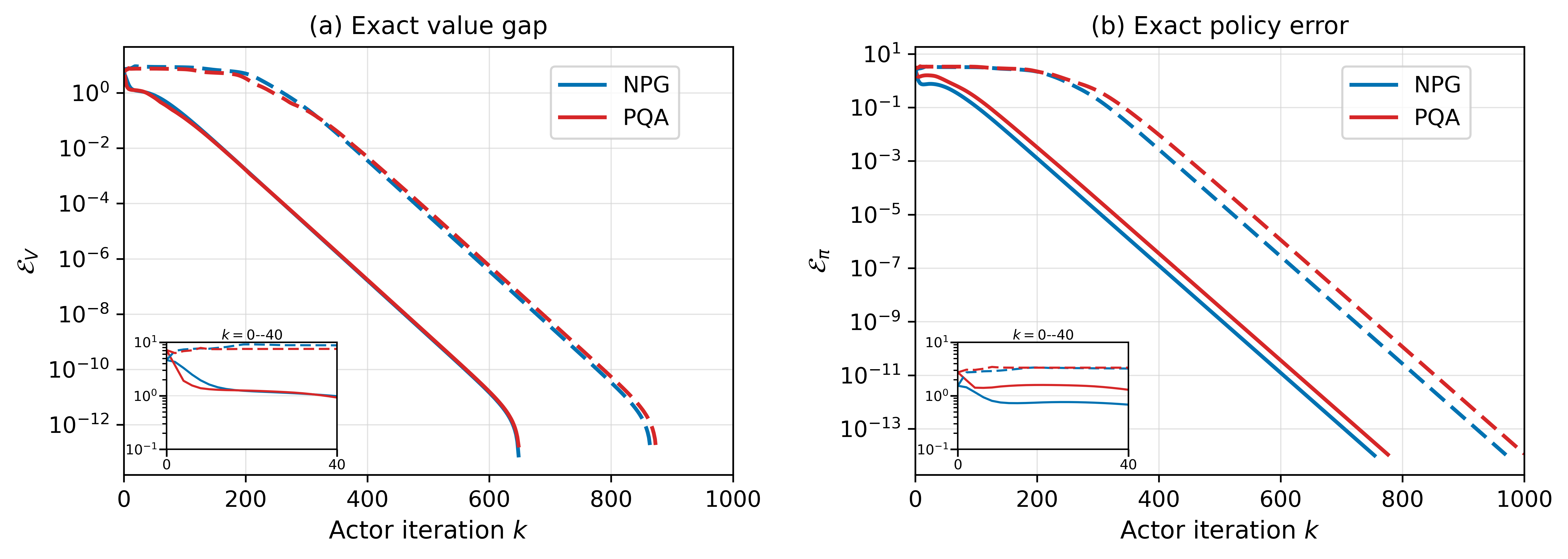}
\caption{Exact coordinate-wise TD--PMD with different critic initializations. Solid curves correspond to initialization that satisfies the one-sided Bellman condition $\calF_\tau^{\pi_0}Q_\tau^0-Q_\tau^0\geq0$ while dashed curves correspond to general initialization. }
\label{fig:numerical-exact-coordinate-comparison}
\end{figure}

It can be observed that both initializations lead to approximately geometric decay of the value gap and policy error after an initial transient. The initialization that violates the one-sided Bellman condition leads to an initial increase in error before the subsequent geometric decay. These observations are consistent with Theorem~\ref{thm:diagonal-global-linear-convergence} and Corollary~\ref{cor:diagonal-policy-convergence}, which control the value gap and policy error, respectively, for arbitrary finite critic initializations. It is also worth noting that the one-sided Bellman condition does not guarantee a monotonic decrease in the policy error, as illustrated in  Figure~\ref{fig:numerical-exact-coordinate-comparison} (right).

\subsection{Finite-batch stochastic experiment}
\label{subsec:numerical-stochastic}

The stochastic experiment uses the same problem setup as the exact experiment. We use sampled transitions with $Q_\tau^0=0$ and set $\gamma=0.5$, $\tau=0.7$, $\eta=4\times10^{-7}$, $\alpha=1$, and $B=10$. The batch weights are given by \eqref{eq:exponential-batch-weights-value-section} with $\vartheta=0.1$. For $x\in\{V,\pi\}$, each run reports the conditional average
\begin{equation}
\widehat{\mathcal E}_{x}(K)
=\frac{\sum_{k=0}^{K-1}\rho^{-k}\mathcal E_x(\pi_k)}
{\sum_{k=0}^{K-1}\rho^{-k}},
\qquad
\rho:=(1+\eta\tau)^{-1}.
\label{eq:numerical-stochastic-output-metric}
\end{equation}
For each regularizer, we perform five independent runs for \(K=5\times10^7\) iterations (\(5\times10^8\) transitions per run). The error are evaluated and accumulated at every iteration and the resulting weighted metrics are recorded every \(10^5\) iterations. Figure~\ref{fig:random-mdp-stochastic-m5} reports the five-run mean together with the range of the five trajectories.

\begin{figure}[ht!]
\centering
\includegraphics[width=0.94\linewidth]{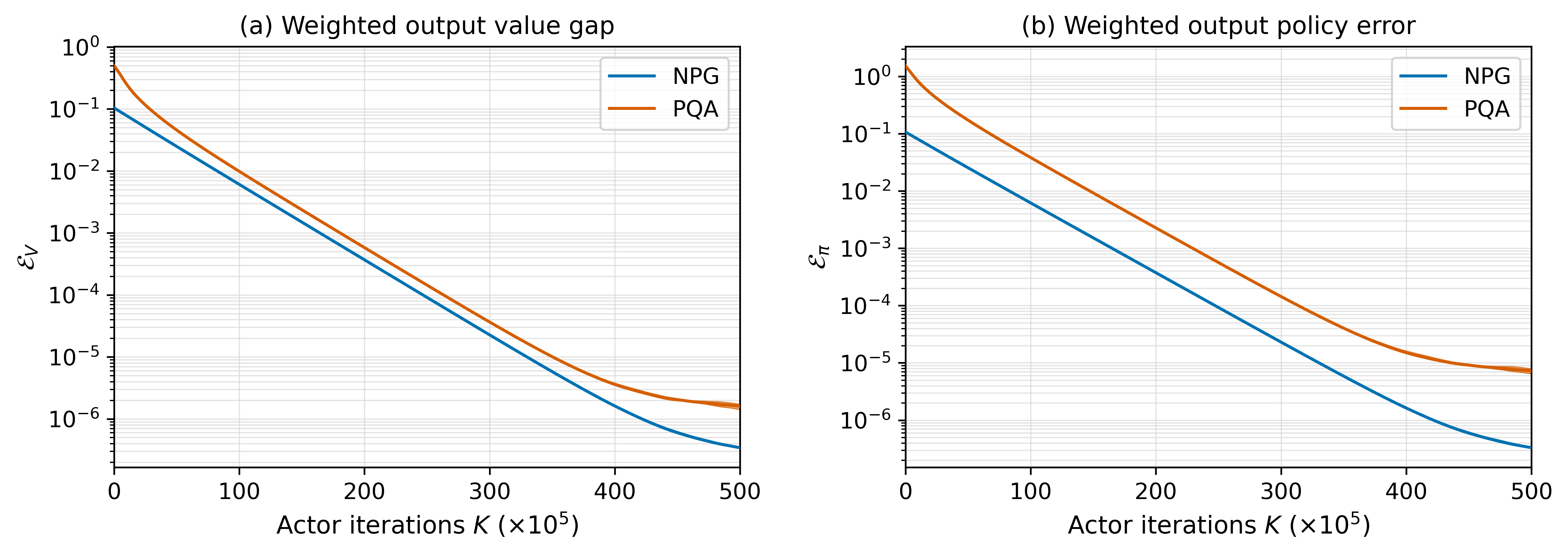}
\caption{Finite-batch stochastic TD--PMD over five independent Markov trajectories.}
\label{fig:random-mdp-stochastic-m5}
\end{figure}

The value metric shows an initial approximately exponential decrease, followed by slower improvement, which is qualitatively consistent with the three-term bound \eqref{eq:explicit-three-term-value-gap} of Theorem~\ref{thm:off-policy-markov-value-gap}.  The policy metric exhibits a similar empirical pattern.  The five independent runs remain close to each other throughout the experiment.

\section{Conclusion and Future Directions}
\label{sec:conclusion}

The present paper develops a convergence theory for regularized TD--PMD with a persistent critic updated through one-step TD recursions.  In the exact setting, we establish global linear convergence for coordinate-wise Bellman updates under general convex mirror geometry.  We then extend the analysis to a single off-policy Markov trajectory and show that, under strong convexity, finite-batch TD--PMD attains an expected value gap of at most $\epsilon$ with $\widetilde{\mathcal O}(\epsilon^{-1})$ observed transitions.  Numerical experiments illustrate the predicted behavior in both the exact and Markov-sampling settings.

Several questions remain open.  First, it is interesting to analyze an online variant that performs one critic update and one policy update for every observed transition. Such a result would require controlling the interaction between Markov dependence, critic noise, and policy drift without relying on within-batch mixing. Second, one may consider a stochastic variant in which the conditional expectation over actions in the TD target is replaced by a single action sampled from the current target policy. In this case, the resulting additional sampling noise needs to be controlled jointly with the Markov dependence and the drift of the target policy. It would also be interesting to develop stochastic TD--PMD schemes with more flexible output rules and stepsize schedules, as well as parameter choices that require less prior knowledge of the behavior-chain mixing and coverage properties.

\bibliography{refs}
\bibliographystyle{plainnat}

%\clearpage
\appendix
\section{Proofs for Section~\ref{sec:exact-diagonal-td-pmd}}
\label{app:exact-diagonal-proof}

\subsection{Proof of Lemma~\ref{lem:diagonal-bellman-violation}}
By setting $p=\pi_k(\cdot\mid s)$, Lemma~\ref{lem:three-point-descent} implies the following actor-improvement inequality
\begin{equation}
\calF_\tau^{\pi_{k+1}}Q_\tau^k
\geq\calF_\tau^{\pi_k}Q_\tau^k.
\label{eq:diagonal-actor-improvement}
\end{equation}
Let $R_k:=\calF_\tau^{\pi_k}Q_\tau^k-Q_\tau^k$ and $P_{k+1}:=P_{\scriptscriptstyle\calS\times\calA}^{\pi_{k+1}}$, and set $\widehat R_k:=\calF_\tau^{\pi_{k+1}}Q_\tau^k-Q_\tau^k$. By \eqref{eq:diagonal-actor-improvement}, there holds $\widehat R_k\geq R_k$. Additionally, the matrix $I-W+\gamma P_{k+1}W$ is nonnegative since $0<W\leq I$. Leveraging the affine identity $\calF_\tau^{\pi_{k+1}}Q'-\calF_\tau^{\pi_{k+1}}Q =\gamma P_{k+1}(Q'-Q)$ and the critic update $Q_\tau^{k+1}-Q_\tau^k=W\widehat R_k$, we obtain that 
\begin{align}
R_{k+1}
&=\widehat R_k-(Q_\tau^{k+1}-Q_\tau^k)
+\gamma P_{k+1}(Q_\tau^{k+1}-Q_\tau^k)
\nonumber\\
&=\widehat R_k-W\widehat R_k+\gamma P_{k+1}W\widehat R_k
\nonumber\\
&=(I-W+\gamma P_{k+1}W)\widehat R_k
\nonumber\\
&\geq(I-W+\gamma P_{k+1}W)R_k.
\label{eq:diagonal-residual-recursion}
\end{align}
Taking positive parts after multiplying by $-1$,
\begin{equation}
[-R_{k+1}]_+\leq(I-W+\gamma P_{k+1}W)[-R_k]_+.
\label{eq:diagonal-positive-part-recursion}
\end{equation}
Since $W$ is diagonal, we have $W(I-W+\gamma P_{k+1}W)=(I-W+\gamma WP_{k+1})W$, and the row sum of $I-W+\gamma WP_{k+1}$ corresponding to coordinate $(s,a)$ is $1-w(s,a)(1-\gamma)\leq1-\underline w(1-\gamma)$.  Therefore
\[
\|W[-R_{k+1}]_+\|_\infty
\leq[1-\underline w(1-\gamma)]\|W[-R_k]_+\|_\infty,
\]
which proves $\delta_{k+1}\leq[1-\underline w(1-\gamma)]\delta_k$.

We next prove the critic comparisons in \eqref{eq:diagonal-violation-conclusions}.  By the definition of $\delta_k$, for every $(s,a)$,
\[
w(s,a)[-R_k(s,a)]_+
\leq\|W[-R_k]_+\|_\infty
=\underline w(1-\gamma)\delta_k,
\]
and hence
\[
[-R_k(s,a)]_+
\leq\frac{\underline w}{w(s,a)}(1-\gamma)\delta_k
\leq(1-\gamma)\delta_k.
\]
Thus $R_k\geq-(1-\gamma)\delta_k\mathbf1$.  By the definition of the Bellman operator,
\[
\calF_\tau^{\pi_k}(Q_\tau^k-\delta_k\mathbf1)
-(Q_\tau^k-\delta_k\mathbf1)
=R_k+(1-\gamma)\delta_k\mathbf1\geq0.
\]
Iterating the monotone $\gamma$-contraction operator $\calF_\tau^{\pi_k}$ to its fixed point proves $Q_\tau^k-\delta_k\mathbf1\leq Q_\tau^{\pi_k}$, while $Q_\tau^{\pi_k}\leq Q_\tau^*$ follows from the optimality of $Q^*_\tau$.  Finally, \eqref{eq:diagonal-actor-improvement} implies
\[
\calF_\tau^{\pi_{k+1}}(Q_\tau^k-\delta_k\mathbf1)
-(Q_\tau^k-\delta_k\mathbf1)
=\widehat R_k+(1-\gamma)\delta_k\mathbf1
\geq R_k+(1-\gamma)\delta_k\mathbf1\geq0.
\]
Iterating $\calF_\tau^{\pi_{k+1}}$ to its fixed point gives the inequality $Q_\tau^k-\delta_k\mathbf1\leq Q_\tau^{\pi_{k+1}}$, which completes the proof.

%%%%%%%%%%%%%%%%%%

\subsection{Proof of Lemma~\ref{lem:diagonal-resolvent-comparison}}
Since $\gamma/\xi\in(0,1)$, the distribution $\nu_{\mu,\xi}^*$ is a convex combination of probability distributions.  Shifting its series by one index gives
\[
\frac{\gamma}{\xi}(\nu_{\mu,\xi}^*)^\top
P_{\scriptscriptstyle\calS}^{\pi_\tau^*}
=(\nu_{\mu,\xi}^*)^\top-
\left(1-\frac{\gamma}{\xi}\right)\mu^\top.
\]
Multiplying by $\xi$ proves the first identity in \eqref{eq:diagonal-resolvent-comparison}.

Let $m:=\min_{s:\,\nu_{\mu,\xi}^*(s)>0} \mu(s)/\nu_{\mu,\xi}^*(s) \geq 0$. Noting that the support of $\mu$ is contained in that of $\nu_{\mu,\xi}^*$, there holds
\[
1=\sum_{s:\,\nu_{\mu,\xi}^*(s)>0}\nu_{\mu,\xi}^*(s)
\frac{\mu(s)}{\nu_{\mu,\xi}^*(s)}
\geq m\sum_{s:\,\nu_{\mu,\xi}^*(s)>0}\nu_{\mu,\xi}^*(s)=m,
\]
so $m\leq1$.  Since $\gamma_{\mu,\xi}=\xi-(\xi-\gamma)m$, we have
$(\xi-\gamma)\mu(s)
\geq(\xi-\gamma)m\nu_{\mu,\xi}^*(s)
=(\xi-\gamma_{\mu,\xi})\nu_{\mu,\xi}^*(s)$
for every state.  This also proves $\gamma_{\mu,\xi}\in[\gamma,\xi]$.

%%%%%%%%%%%%%%%%

\subsection{Proof of Lemma~\ref{lem:diagonal-weighted-recursion}}
For brevity, define the first term of \eqref{eq:diagonal-arbitrary-mu-potential} by
\[
\mathcal B_k:=
\sum_{s,a}\frac{\nu_{\mu,\xi}^*(s)\pi_\tau^*(a\mid s)}{w(s,a)}
\bigl(Q_\tau^*(s,a)-Q_\tau^k(s,a)+\delta_k\bigr).
\]
Lemma~\ref{lem:diagonal-bellman-violation} implies $\mathcal B_k\geq0$. We first derive a critic inequality for an arbitrary fixed $\beta \in [ \gamma_{\mu,\xi}, 1)$ and specify its value at the end of the proof.

\paragraph{Step 1: Policy estimate.}
Define the value estimate
$V_\tau^{k+1}(s):=
\langle\pi_{k+1}(\cdot\mid s),Q_\tau^k(s,\cdot)\rangle
-\tau h^{\pi_{k+1}}(s)$.
Expanding the two values statewise, applying Lemma~\ref{lem:three-point-descent} with $p=\pi_\tau^*(\cdot\mid s)$, and immediately dropping the nonnegative term $D_{\pi_k}^{\pi_{k+1}}(s)$ gives
\begin{equation}
V_\tau^*(\nu_{\mu,\xi}^*)-V_\tau^{k+1}(\nu_{\mu,\xi}^*)
+(\eta^{-1}+\tau)
D_{\pi_{k+1}}^{\pi_\tau^*}(\nu_{\mu,\xi}^*)
\leq
\E_{\substack{s\sim\nu_{\mu,\xi}^*\\
a\sim\pi_\tau^*(\cdot\mid s)}}
\!\left[Q_\tau^*(s,a)-Q_\tau^k(s,a)\right]
+
\eta^{-1}D_{\pi_k}^{\pi_\tau^*}(\nu_{\mu,\xi}^*).
\label{eq:diagonal-three-point-value-bound}
\end{equation}

\paragraph{Step 2: Distribution comparison.}
By Lemma~\ref{lem:diagonal-bellman-violation} we have $Q_\tau^k-\delta_k\mathbf1\leq Q_\tau^{\pi_{k+1}}$.  Hence, by the definition of $V_\tau^{k+1}$,
\begin{align*}
\forall\, s\in\mathcal{S}: \;\; V_\tau^{k+1}(s)-\delta_k
&=\left\langle\pi_{k+1}(\cdot\mid s),
Q_\tau^k(s,\cdot)-\delta_k\mathbf1\right\rangle
-\tau h^{\pi_{k+1}}(s)\\
&\leq
\left\langle\pi_{k+1}(\cdot\mid s),
Q_\tau^{\pi_{k+1}}(s,\cdot)\right\rangle
-\tau h^{\pi_{k+1}}(s)
=V_\tau^{\pi_{k+1}}(s)
\leq V_\tau^*(s).
\end{align*}
Thus $V_\tau^*-V_\tau^{k+1}+\delta_k\mathbf1\geq0$. Together with the second componentwise relation in~\eqref{eq:diagonal-resolvent-comparison} and the fact that $(\xi - \gamma) \mu(s) \ge (\xi - \gamma_{\mu,\xi}) \nu_{\mu,\xi}^*(s) \ge (\xi - \beta) \nu_{\mu,\xi}^*(s)$, this gives
\begin{equation}
(\xi-\gamma)
\left[V_\tau^*(\mu)-V_\tau^{k+1}(\mu)\right]
\geq(\xi-\beta)
\left[V_\tau^*(\nu_{\mu,\xi}^*)
-V_\tau^{k+1}(\nu_{\mu,\xi}^*)\right]
-(\beta-\gamma)\delta_k.
\label{eq:diagonal-arbitrary-mu-overlap-comparison}
\end{equation}
By the Bellman equation,
\[
\forall\, (s,a) \in \mathcal{S}\times\mathcal{A}: \quad Q_\tau^*(s,a)-\calF_\tau^{\pi_{k+1}}Q_\tau^k(s,a)
=\gamma\sum_{s'}P(s'\mid s,a)
\bigl[V_\tau^*(s')-V_\tau^{k+1}(s')\bigr].
\]
Taking expectation with respect to $\nu_{\mu,\xi}^*(s)\pi_\tau^*(a\mid s)$ and then applying the first identity in \eqref{eq:diagonal-resolvent-comparison}, we obtain
\begin{align*}
&\E_{\substack{s\sim\nu_{\mu,\xi}^*\\
a\sim\pi_\tau^*(\cdot\mid s)}}
\!\left[\calF_\tau^{\pi_{k+1}}Q_\tau^k(s,a)-Q_\tau^k(s,a)\right]
=
\E_{\substack{s\sim\nu_{\mu,\xi}^*\\
a\sim\pi_\tau^*(\cdot\mid s)}}
\!\left[Q_\tau^*(s,a)-Q_\tau^k(s,a)\right]
 -\gamma(\nu_{\mu,\xi}^*)^\top
 P_{\scriptscriptstyle\calS}^{\pi_\tau^*}
 (V_\tau^*-V_\tau^{k+1})
 \\
 &=
 \E_{\substack{s\sim\nu_{\mu,\xi}^*\\
 a\sim\pi_\tau^*(\cdot\mid s)}}
 \!\left[Q_\tau^*(s,a)-Q_\tau^k(s,a)\right]
-\xi\left[V_\tau^*(\nu_{\mu,\xi}^*)
-V_\tau^{k+1}(\nu_{\mu,\xi}^*)\right]
+(\xi-\gamma)
\left[V_\tau^*(\mu)-V_\tau^{k+1}(\mu)\right] \\
&\geq \E_{\substack{s\sim\nu_{\mu,\xi}^*\\
 a\sim\pi_\tau^*(\cdot\mid s)}}
 \!\left[Q_\tau^*(s,a)-Q_\tau^k(s,a)\right]
-\xi\left[V_\tau^*(\nu_{\mu,\xi}^*)
-V_\tau^{k+1}(\nu_{\mu,\xi}^*)\right] + (\xi-\beta)
\left[V_\tau^*(\nu_{\mu,\xi}^*)
-V_\tau^{k+1}(\nu_{\mu,\xi}^*)\right]
-(\beta-\gamma)\delta_k.
\end{align*}
The last inequality follows from
\eqref{eq:diagonal-arbitrary-mu-overlap-comparison}.
Substituting \eqref{eq:diagonal-three-point-value-bound} yields
\begin{align}
\E_{\substack{s\sim\nu_{\mu,\xi}^*\\
a\sim\pi_\tau^*(\cdot\mid s)}}
\!\left[
\calF_\tau^{\pi_{k+1}}Q_\tau^k(s,a)-Q_\tau^k(s,a)
\right]
&\geq(1-\beta)
\E_{\substack{s\sim\nu_{\mu,\xi}^*\\
a\sim\pi_\tau^*(\cdot\mid s)}}
\!\left[Q_\tau^*(s,a)-Q_\tau^k(s,a)\right]
\nonumber\\
&\quad
-\beta\eta^{-1}
D_{\pi_k}^{\pi_\tau^*}(\nu_{\mu,\xi}^*)
+\beta(\eta^{-1}+\tau)
D_{\pi_{k+1}}^{\pi_\tau^*}(\nu_{\mu,\xi}^*)
-(\beta-\gamma)\delta_k.
\label{eq:diagonal-general-critic-recursion}
\end{align}

\paragraph{Step 3: Critic update.}
The inverse weight in $\mathcal B_k$ cancels the corresponding coordinate weight $w(s,a)$ in the critic update.  Thus
\begin{equation}
\mathcal B_{k+1}
=\mathcal B_k
-\E_{\substack{s\sim\nu_{\mu,\xi}^*\\
a\sim\pi_\tau^*(\cdot\mid s)}}
\!\left[\calF_\tau^{\pi_{k+1}}Q_\tau^k(s,a)-Q_\tau^k(s,a)\right]
+(\delta_{k+1}-\delta_k)
\sum_{s,a}\frac{\nu_{\mu,\xi}^*(s)\pi_\tau^*(a\mid s)}{w(s,a)}.
\label{eq:diagonal-general-critic-update}
\end{equation}
Since every shifted critic coordinate is nonnegative and $w(s,a)\geq\underline w$,
\begin{equation}
0\leq\underline w\mathcal B_k
\leq
\E_{\substack{s\sim\nu_{\mu,\xi}^*\\
a\sim\pi_\tau^*(\cdot\mid s)}}
\!\left[Q_\tau^*(s,a)-Q_\tau^k(s,a)\right]+\delta_k.
\label{eq:diagonal-general-lower-weight-comparison}
\end{equation}
Substituting \eqref{eq:diagonal-general-critic-recursion} into \eqref{eq:diagonal-general-critic-update}, using $\delta_{k+1}-\delta_k\leq-\underline w(1-\gamma)\delta_k$, and then applying \eqref{eq:diagonal-general-lower-weight-comparison}, we obtain
\begin{align}
\mathcal B_{k+1}
&\leq
[1-\underline w(1-\beta)]\mathcal B_k
+\beta\eta^{-1}
D_{\pi_k}^{\pi_\tau^*}(\nu_{\mu,\xi}^*)
-\beta(\eta^{-1}+\tau)
D_{\pi_{k+1}}^{\pi_\tau^*}(\nu_{\mu,\xi}^*)
\nonumber\\
&\quad
+(1-\gamma)\left(
1-\underline w\sum_{s,a}
\frac{\nu_{\mu,\xi}^*(s)\pi_\tau^*(a\mid s)}{w(s,a)}
\right)\delta_k.
\label{eq:diagonal-general-corrected-recursion}
\end{align}
Now choosing
\[
\beta = \max \{ \gamma_{\mu,\xi}, (1+\eta\tau)^{-1}\} \in [\gamma_{\mu,\xi},1),
\]
we obtain $\beta(\eta^{-1}+ \tau) = \max \{\eta^{-1}, \gamma_{\mu,\xi}(\eta^{-1} +\tau) \}$ and $1 - \underline w (1 - \beta) = \rho$. Thus the potential in \eqref{eq:diagonal-arbitrary-mu-potential} is
$\mathcal L_k
=\mathcal B_k
+\beta(\eta^{-1}+\tau)
D_{\pi_k}^{\pi_\tau^*}(\nu_{\mu,\xi}^*)$.
Adding the Bregman component of $\mathcal L_{k+1}$ to both sides of \eqref{eq:diagonal-general-corrected-recursion}, the definition of $\mathcal L_k$ and the fact that $(1+\eta\tau)^{-1} \le \rho$ gives that
\[
\mathcal L_{k+1}\leq\rho\mathcal L_k
+(1-\gamma)\left(1-\underline w\sum_{s,a}
\frac{\nu_{\mu,\xi}^*(s)\pi_\tau^*(a\mid s)}{w(s,a)}\right)\delta_k,
\]
which completes the proof.
%%%%%%%%%%%%%%%%%%

\subsection{Proof of Theorem~\ref{thm:diagonal-global-linear-convergence}}
By Lemma~\ref{lem:diagonal-bellman-violation} there holds $\delta_k\leq[1-\underline w(1-\gamma)]^k\delta_0$. Iterating \eqref{eq:diagonal-exact-energy-recursion} and using the definition of $\mathcal C_k$ yields 
\begin{equation}
\mathcal L_k\leq\mathcal C_k.
\label{eq:diagonal-iterated-potential-bound}
\end{equation}
The first term of the resolvent series shows that $\nu_{\mu,\xi}^*$ assigns positive mass wherever $\mu(s) > 0$.  Hence, for every nonnegative $f:\calS\to\mathbb R$, there holds
\[
\mu^\top f
\leq\left\|\frac{\mu}{\nu_{\mu,\xi}^*}\right\|_\infty
(\nu_{\mu,\xi}^*)^\top f.
\]
Write the first term of $\mathcal L_k$ as $\mathcal B_k$. The policy estimate \eqref{eq:diagonal-three-point-value-bound} and the comparison $V_\tau^{\pi_{k+1}}\geq V_\tau^{k+1}-\delta_k\mathbf1$ imply that
\begin{align}
V_\tau^*(\nu_{\mu,\xi}^*)
-V_\tau^{\pi_{k+1}}(\nu_{\mu,\xi}^*)
&\leq V_\tau^*(\nu_{\mu,\xi}^*)
-V_\tau^{k+1}(\nu_{\mu,\xi}^*)+\delta_k
\nonumber\\
&\overset{(a)}{\leq}
\E_{\substack{s\sim\nu_{\mu,\xi}^*\\
a\sim\pi_\tau^*(\cdot\mid s)}}
\!\left[Q_\tau^*(s,a)-Q_\tau^k(s,a)+\delta_k\right]
+\eta^{-1}D_{\pi_k}^{\pi_\tau^*}(\nu_{\mu,\xi}^*)
\nonumber\\
&\overset{(b)}{\leq}\overline w\mathcal B_k
+\eta^{-1}D_{\pi_k}^{\pi_\tau^*}(\nu_{\mu,\xi}^*)
\nonumber\\
&\overset{(c)}{\leq}
\mathcal L_k.
\label{eq:diagonal-resolvent-value-comparison}
\end{align}
Here (a) follows from \eqref{eq:diagonal-three-point-value-bound} after dropping the nonnegative Bregman term, (b) uses $Q_\tau^*-Q_\tau^k+\delta_k\mathbf1\geq0$ and $w(s,a)\leq\overline w$, and (c) follows directly from $\overline w \le 1$ and the definition of $\mathcal L_k$. Since the value gap is pointwise nonnegative, we apply the preceding inequality for nonnegative $f$ with
$f=V_\tau^*-V_\tau^{\pi_{k+1}}$. Combining the resulting bound with \eqref{eq:diagonal-resolvent-value-comparison} and \eqref{eq:diagonal-iterated-potential-bound} proves \eqref{eq:diagonal-arbitrary-mu-value-rate}.
Moreover, stationarity gives $\nu_{\nu^*,\xi}^*=\nu^*$ for every $\xi\in(\gamma,1)$, which implies $\gamma_{\nu^*,\xi}=\gamma$ and $\|\nu^*/\nu_{\nu^*,\xi}^*\|_\infty=1$. Substituting these identities into the definition of $\mathcal C_k$ gives the stated value bound and contraction factor.

%%%%%%%%%%%%%%%%%%

\subsection{Proof of Corollary~\ref{cor:unregularized-exact-td-pmd}}
Let
\[
B_k:=\sum_{s,a}\nu^*(s)\pi^*(a\mid s)
\bigl(Q^*(s,a)-Q^k(s,a)+\delta_k\bigr)\geq0.
\]
Taking $\mu=\nu^*$, $\beta=\gamma$, $\tau=0$, and $W=I$
in the first inequality of
\eqref{eq:diagonal-general-corrected-recursion},
and replacing $\eta$ by $\eta_k>0$, yields
\begin{equation}
B_{k+1}+\frac{\gamma}{\eta_k}
D_{\pi_{k+1}}^{\pi^*}(\nu^*)
\leq
\gamma\left[
B_k+\frac{1}{\eta_k}D_{\pi_k}^{\pi^*}(\nu^*)
\right].
\label{eq:unregularized-potential-descent}
\end{equation}

The same three-point argument, together with
$V^{\pi_{k+1}}\geq V^{k+1}-\delta_k\mathbf1$, yields
\begin{equation}
V^*(\nu^*)-V^{\pi_{k+1}}(\nu^*)
+\eta_k^{-1}D_{\pi_{k+1}}^{\pi^*}(\nu^*)
\leq B_k+\eta_k^{-1}D_{\pi_k}^{\pi^*}(\nu^*).
\label{eq:unregularized-value-descent}
\end{equation}
When $\eta_k=\eta$, summing \eqref{eq:unregularized-potential-descent} from $k=0$ to $K-1$ provides the bound for $\sum_{k=0}^{K-1}B_k$. Furthermore, summing \eqref{eq:unregularized-value-descent} and then telescoping the Bregman terms gives \eqref{eq:unregularized-constant-step-rate}.

Now consider the increasing-stepsize case and set $\Phi_k:=B_k+\eta_k^{-1}D_{\pi_k}^{\pi^*}(\nu^*)$.  Since $\eta_{k+1}^{-1}\leq\gamma\eta_k^{-1}$, \eqref{eq:unregularized-potential-descent} gives $\Phi_{k+1}\leq\gamma\Phi_k$.  Moreover, \eqref{eq:unregularized-value-descent} gives $V^*(\nu^*)-V^{\pi_{k+1}}(\nu^*)\leq\Phi_k$. Therefore $\Phi_k\leq\gamma^k\Phi_0$, which proves \eqref{eq:unregularized-increasing-step-rate}.

%%%%%%%%%%%%%%%%%%

\section{Proofs for Section~\ref{sec:off-policy-markov-data}}
\label{app:markov-proof}

%\subsection{Proofs for Lemmas~\ref{lem:bounded-stochastic-critic-iterates}--\ref{lem:stochastic-error-bound}}

\subsection{Proof of Lemma~\ref{lem:bounded-stochastic-critic-iterates}}
For each sample, define its TD target by
\[
Y_t^k
:=
r_t^k
+\gamma\mathbb E_{a'\sim\pi_{k+1}(\cdot\mid s_{t+1}^k)}
\left[Q_\tau^k(s_{t+1}^k,a')\right]
-\tau\gamma h^{\pi_{k+1}}(s_{t+1}^k).
\]
For each $(s,a)$, let
\[
\widehat\sigma_k(s,a)
:=
\sum_{t=0}^{B_k-1}
c_t^k\mathds{1}[(s_t^k,a_t^k)=(s,a)]\in[0,1].
\]
The normalized weighted target is then defined as
\[
\widehat y_k(s,a)
:=
\begin{cases}
\displaystyle
\frac{1}{\widehat\sigma_k(s,a)}
\sum_{t=0}^{B_k-1}
c_t^k\mathds{1}[(s_t^k,a_t^k)=(s,a)]Y_t^k,
& \widehat\sigma_k(s,a)>0,\\[2ex]
0,
& \widehat\sigma_k(s,a)=0.
\end{cases}
\]
With this notation, the coordinate critic update becomes
\[
Q_\tau^{k+1}(s,a)
=
\bigl(1-\alpha_k\widehat\sigma_k(s,a)\bigr)Q_\tau^k(s,a)
+
\alpha_k\widehat\sigma_k(s,a)\widehat y_k(s,a).
\]
If
$\|Q_\tau^k\|_\infty\leq(1+\tau\gamma H_h)/(1-\gamma)$,
Assumptions~\ref{ass:normalization} and \ref{ass:h} give
\[
|\widehat y_k(s,a)|
\leq
1+\gamma\frac{1+\tau\gamma H_h}{1-\gamma}
+\tau\gamma H_h
=
\frac{1+\tau\gamma H_h}{1-\gamma}.
\]
Thus $Q_\tau^{k+1}(s,a)$ is a convex combination of two points in
\(
[-(1+\tau\gamma H_h)/(1-\gamma),
(1+\tau\gamma H_h)/(1-\gamma)]
\).
Induction from $Q_\tau^0=0$ proves the claim.
%%%%%%%%%%%%%%%%%

\subsection{Proof of Lemma~\ref{lem:trajectory-regularity}}
%We proceed in four steps.

\paragraph{Step 1: Bounded vertex divergences.}
Fix $k\geq0$, $s\in\calS$, and $a\in\calA$. By Lemma~\ref{lem:three-point-descent} with $p=e_a$,
\begin{align*}
&
\left\langle
\pi_{k+1}(\cdot\mid s)-e_a,Q_\tau^k(s,\cdot)
\right\rangle
-\tau\left[h^{\pi_{k+1}}(s)-h(e_a)\right]
\\
&\qquad\geq
\eta^{-1}\!\left[
D_h\!\left(\pi_{k+1}(\cdot\mid s)\,\middle\|\,
\pi_k(\cdot\mid s)\right)
+(1+\eta\tau)D_h\!\left(e_a\,\middle\|\,
\pi_{k+1}(\cdot\mid s)\right)
-D_h\!\left(e_a\,\middle\|\,\pi_k(\cdot\mid s)\right)
\right].
\end{align*}
Using Lemma~\ref{lem:bounded-stochastic-critic-iterates} to bound the critic,
\begin{equation}
(1+\eta\tau)D_h\!\left(e_a\,\middle\|\,
\pi_{k+1}(\cdot\mid s)\right)
\leq
D_h\!\left(e_a\,\middle\|\,\pi_k(\cdot\mid s)\right)
+\frac{2\eta(1+\tau H_h)}{1-\gamma},
\label{eq:vertex-divergence-one-step}
\end{equation}
where we drop the nonnegative divergence term $D_h(\pi_{k+1}(\cdot|s) \, \| \, \pi_k(\cdot|s))$ and use the bounds  $\|\pi_{k+1}(\cdot\mid s)-e_a\|_1\leq2$ and $|h|\leq H_h$ on $\Delta(\calA)$. Equivalently,
\[
D_h\!\left(e_a\,\middle\|\,\pi_{k+1}(\cdot\mid s)\right)
\leq
\frac{D_h\!\left(e_a\,\middle\|\,\pi_k(\cdot\mid s)\right)}
{1+\eta\tau}
+\left(1-\frac{1}{1+\eta\tau}\right)
\frac{2(1+\tau H_h)}{\tau(1-\gamma)}.
\]
The right-hand side is a convex combination of the preceding divergence and the displayed constant. By induction we have
\[
D_h\!\left(e_a\,\middle\|\,\pi_k(\cdot\mid s)\right)
\leq
\max\left\{
\max_{s'\in\calS,\,b\in\calA}
D_h\!\left(e_b\,\middle\|\,\pi_0(\cdot\mid s')\right),
\frac{2(1+\tau H_h)}{\tau(1-\gamma)}
\right\}.
\]
The initial divergences are finite because Assumption~\ref{ass:h} ensures $e_a\in\dom\, h$, and Assumption~\ref{ass:policy-initialization} ensures $\pi_0(\cdot\mid s)\in\rint(\dom\, h)$.

\paragraph{Step 2: Trajectory-wise Lipschitz continuity.}
For any $a,b\in\calA$ we have
\[
\nabla_a h\!\left(\pi_k(\cdot\mid s)\right)
-\nabla_b h\!\left(\pi_k(\cdot\mid s)\right)
=
h(e_a)-h(e_b)
-D_h\!\left(e_a\,\middle\|\,\pi_k(\cdot\mid s)\right)
+D_h\!\left(e_b\,\middle\|\,\pi_k(\cdot\mid s)\right).
\]
The vertex divergences are nonnegative and satisfy the preceding bound as $|h(e_a)-h(e_b)|\leq2H_h$.  Thus,
\begin{equation}
\max_{a\in\calA}\nabla_a h\!\left(\pi_k(\cdot\mid s)\right)
-\min_{a\in\calA}\nabla_a h\!\left(\pi_k(\cdot\mid s)\right)
\leq
2H_h+\max\left\{
\max_{s'\in\calS,\,c\in\calA}
D_h\!\left(e_c\,\middle\|\,\pi_0(\cdot\mid s')\right),
\frac{2(1+\tau H_h)}{\tau(1-\gamma)}
\right\}
=
2L_h.
\label{eq:mirror-gradient-oscillation}
\end{equation}
For $p\in\{\pi_k(\cdot\mid s),\pi_{k+1}(\cdot\mid s)\}$, let
$c_p:=\frac12[\max_{a\in\calA}\nabla_a h(p)
+\min_{a\in\calA}\nabla_a h(p)]$.
Leveraging the convexity of $h$ and \eqref{eq:mirror-gradient-oscillation} yields
\[
h^{\pi_{k+1}}(s)-h^{\pi_k}(s)
\leq
\left\langle
\nabla h\!\left(\pi_{k+1}(\cdot\mid s)\right)
-c_{\pi_{k+1}(\cdot\mid s)}\mathbf1,
\pi_{k+1}(\cdot\mid s)-\pi_k(\cdot\mid s)
\right\rangle
\leq L_h\|\pi_{k+1}(\cdot\mid s)-\pi_k(\cdot\mid s)\|_1.
\]
Interchanging the two policies gives the reverse inequality and proves
\eqref{eq:h-trajectory-lipschitz-proved}.

\paragraph{Step 3: Policy increments.}
By setting $p = \pi_k(\cdot\mid s)$, Lemma~\ref{lem:three-point-descent} implies that
\begin{align*}
&
\left\langle
\pi_{k+1}(\cdot\mid s)-\pi_k(\cdot\mid s),Q_\tau^k(s,\cdot)
\right\rangle
-\tau\left[h^{\pi_{k+1}}(s)-h^{\pi_k}(s)\right]
\\
&\qquad\geq
\eta^{-1}\!\left[
D_h\!\left(\pi_{k+1}(\cdot\mid s)\,\middle\|\,
\pi_k(\cdot\mid s)\right)
+(1+\eta\tau)D_h\!\left(\pi_k(\cdot\mid s)\,\middle\|\,
\pi_{k+1}(\cdot\mid s)\right)
\right].
\end{align*}
Using Assumption~\ref{ass:strong-convexity}, together with H\"older's inequality,
Lemma~\ref{lem:bounded-stochastic-critic-iterates}, and \eqref{eq:h-trajectory-lipschitz-proved}, we obtain
\[
\lambda\|\pi_{k+1}(\cdot\mid s)-\pi_k(\cdot\mid s)\|_1^2
\leq
\eta\left(\frac{1+\tau\gamma H_h}{1-\gamma}+\tau L_h\right)
\|\pi_{k+1}(\cdot\mid s)-\pi_k(\cdot\mid s)\|_1,
\]
from which \eqref{eq:uniform-policy-increment-value-section} follows immediately.

\paragraph{Step 4: Action-value variation.}
For every state $s\in\calS$, the Bellman representation gives
\begin{align}
V_\tau^{\pi_{k+1}}(s)-V_\tau^{\pi_k}(s)
&=\left\langle\pi_{k+1}(\cdot\mid s),
Q_\tau^{\pi_{k+1}}(s,\cdot)-Q_\tau^{\pi_k}(s,\cdot)\right\rangle
\nonumber\\
&\quad+\left\langle\pi_{k+1}(\cdot\mid s)-\pi_k(\cdot\mid s),
Q_\tau^{\pi_k}(s,\cdot)\right\rangle
-\tau\left[h^{\pi_{k+1}}(s)-h^{\pi_k}(s)\right].
\label{eq:value-policy-difference-decomposition}
\end{align}
By H\"older's inequality, Lemma~\ref{lem:bounded_value}, and
\eqref{eq:h-trajectory-lipschitz-proved},
\begin{equation}
\|V_\tau^{\pi_{k+1}}-V_\tau^{\pi_k}\|_\infty
\leq\|Q_\tau^{\pi_{k+1}}-Q_\tau^{\pi_k}\|_\infty
+\left(\frac{1+\tau\gamma H_h}{1-\gamma}+\tau L_h\right)
\|\pi_{k+1}-\pi_k\|_{1,\infty}.
\label{eq:value-policy-difference-bound}
\end{equation}
Additionally, the Bellman equations give
\begin{equation}
\|Q_\tau^{\pi_{k+1}}-Q_\tau^{\pi_k}\|_\infty
\leq\gamma\|V_\tau^{\pi_{k+1}}-V_\tau^{\pi_k}\|_\infty.
\label{eq:q-value-difference}
\end{equation}
Combining the last two inequalities proves \eqref{eq:q-trajectory-lipschitz} by rearranging terms.
%%%%%%%%%%%%%%%%%%%%%%%

\subsection{Proof of Lemma~\ref{lem:stochastic-error-bound}}
Fix $k$, $(s,a)$, and condition on $\mathcal H_k$. By the Markov property and the definition of the regularized TD sample,
\begin{align}
\mathbb E\!\left[\delta_t^k(s,a)\mid\mathcal H_k\right]
&=
\mathbb P(s_t^k=s\mid s_0^k)\pi_b(a\mid s)
\left[
\mathcal F_\tau^{\pi_{k+1}}Q_\tau^k-Q_\tau^k
\right](s,a).
\label{eq:conditional-td-mean}
\end{align}
The corresponding stationary mean of the sampled increment is
\[
\nu^{\pi_b}(s)\pi_b(a\mid s)
\left[\mathcal F_\tau^{\pi_{k+1}}Q_\tau^k-Q_\tau^k\right](s,a)
=
\left[
\Sigma_b\bigl(\mathcal F_\tau^{\pi_{k+1}}Q_\tau^k-Q_\tau^k\bigr)
\right](s,a).
\]
Hence the definition of $\omega_t^k$ yields the exact identity
\begin{align}
\mathbb E\!\left[\omega_t^k(s,a)\mid\mathcal H_k\right]
&=
\left[
\mathbb P(s_t^k=s\mid s_0^k)-\nu^{\pi_b}(s)
\right]\pi_b(a\mid s)
\left[
\mathcal F_\tau^{\pi_{k+1}}Q_\tau^k-Q_\tau^k
\right](s,a).
\label{eq:conditional-omega-identity}
\end{align}
Lemma~\ref{lem:bounded-stochastic-critic-iterates} and Assumptions~\ref{ass:normalization}--\ref{ass:h} imply that both $Q_\tau^k(s,a)$ and $\mathcal F_\tau^{\pi_{k+1}}Q_\tau^k(s,a)$ belong to $[-(1+\tau\gamma H_h)/(1-\gamma), (1+\tau\gamma H_h)/(1-\gamma)]$. Therefore, the mixing estimate in \eqref{eq:behavior-exponential-mixing} gives
\begin{align}
\left|
\mathbb E\!\left[\omega_t^k(s,a)\mid\mathcal H_k\right]
\right|
&\leq
\frac{2(1+\tau\gamma H_h)}{1-\gamma}
d_{\mathrm{TV}}\!\left(
\mathbb P(s_t^k=\cdot\mid s_0^k),\nu^{\pi_b}
\right)
\leq
\frac{2m_b(1+\tau\gamma H_h)}{1-\gamma}\kappa_b^t.
\label{eq:conditional-one-step-critic-bias}
\end{align}
Finally, the nonnegativity of the batch weights and $\bar\omega_k=\sum_{t=0}^{B_k-1}c_t^k\omega_t^k$ imply
\begin{align}
\left\|
\mathbb E\!\left[\bar\omega_k\mid\mathcal H_k\right]
\right\|_\infty
&\leq
\sum_{t=0}^{B_k-1}c_t^k
\left\|
\mathbb E\!\left[\omega_t^k\mid\mathcal H_k\right]
\right\|_\infty
\leq
\frac{2m_b(1+\tau\gamma H_h)}{1-\gamma}
\sum_{t=0}^{B_k-1}c_t^k\kappa_b^t,
\label{eq:conditional-aggregate-critic-bias}
\end{align}
which yields \eqref{eq:stochastic-error-bound}.

%%%%%%%%%%%%%%%%%%%%%%%%%%%%%%%%%
\subsection{Proof of Theorem~\ref{thm:off-policy-markov-value-gap}}

We begin by reducing the performance guarantee to a weighted sum of signed critic errors.  For $s\in\calS$, define
\begin{equation}
\label{eq:Bkk-signed-error}
B_k^{(k)}(s)
:=
\left\langle
\pi_\tau^*(\cdot\mid s)-\pi_k(\cdot\mid s),
Q_\tau^{\pi_k}(s,\cdot)-Q_\tau^k(s,\cdot)
\right\rangle.
\end{equation}

\begin{lemma}[Weighted value-gap reduction]
\label{lem:weighted-value-gap-reduction}
For every $K\geq1$,
\begin{align}
\mathbb E
\left[
V_\tau^*(\mu)-V_\tau^{\pi_{\widehat K}}(\mu)
\right]
&\leq
\frac{\tau D_{\pi_0}^{\pi_\tau^*}(d_\mu^*)}
{(1-\gamma)(\rho^{-K}-1)}
+
\frac{
    2(\rho^{-1}-1)\rho^{-(K-1)}(1+\tau H_h)
}{
(1-\gamma)^2(\rho^{-K}-1)
}
+
    \frac{2\eta(1+\tau\gamma H_h)^2}
    {\lambda(1-\gamma)^4}
\nonumber\\
&\quad+
\frac{\rho^{-1}-1}
{(1-\gamma)(\rho^{-K}-1)}
\sum_{k=0}^{K-1}
\rho^{-k}
\mathbb E
\left[
\mathbb E_{s\sim d_\mu^*}
\left[
B_k^{(k)}(s)
\right]
\right].
\label{eq:weighted-value-gap-master-bound}
\end{align}
\end{lemma}

\begin{proof}
Fix $k\geq0$ and condition on $\mathcal H_k$.  Applying Lemma~\ref{lem:pdl} to $(\pi_\tau^*,\pi_k)$ and then isolating the signed critic error gives
\begin{align}
V_\tau^*(\mu)-V_\tau^{\pi_k}(\mu)
&=\frac{1}{1-\gamma}\mathbb E_{s\sim d_\mu^*}
\left[\left\langle\pi_\tau^*(\cdot\mid s)-\pi_k(\cdot\mid s),
Q_\tau^{\pi_k}(s,\cdot)\right\rangle
-\tau\!\left(h^{\pi_\tau^*}(s)-h^{\pi_k}(s)\right)\right]\nonumber\\
&=\frac{1}{1-\gamma}\mathbb E_{s\sim d_\mu^*}
\left[\left\langle\pi_\tau^*(\cdot\mid s)-\pi_k(\cdot\mid s),
Q_\tau^k(s,\cdot)\right\rangle
-\tau\!\left(h^{\pi_\tau^*}(s)-h^{\pi_k}(s)\right)+B_k^{(k)}(s)\right].
\label{eq:value-gap-pdl-critic-split}
\end{align}
Lemma~\ref{lem:three-point-descent} with $p=\pi_\tau^*(\cdot\mid s)$ yields
\begin{equation}
\eta\!\left[\left\langle\pi_\tau^*(\cdot\mid s)-\pi_{k+1}(\cdot\mid s),
Q_\tau^k(s,\cdot)\right\rangle
-\tau\!\left(h^{\pi_\tau^*}(s)-h^{\pi_{k+1}}(s)\right)\right]
\leq D_{\pi_k}^{\pi_\tau^*}(s)-\rho^{-1}D_{\pi_{k+1}}^{\pi_\tau^*}(s)
-D_{\pi_k}^{\pi_{k+1}}(s).
\label{eq:three-point-optimal-policy-bound}
\end{equation}
By Young's inequality,
\[
\eta\left\|\pi_{k+1}(\cdot\mid s)-\pi_k(\cdot\mid s)\right\|_1
\left\|Q_\tau^k-Q_\tau^{\pi_k}\right\|_\infty
-\frac{\lambda}{2}
\left\|\pi_{k+1}(\cdot\mid s)-\pi_k(\cdot\mid s)\right\|_1^2
\leq\frac{\eta^2}{2\lambda}
\left\|Q_\tau^k-Q_\tau^{\pi_k}\right\|_\infty^2.
\]
Combining this inequality with
\eqref{eq:three-point-optimal-policy-bound},
\begin{align}
&\phantom{=\,\,\,}\left\langle\pi_\tau^*(\cdot\mid s)-\pi_k(\cdot\mid s),
Q_\tau^k(s,\cdot)\right\rangle
-\tau\!\left(h^{\pi_\tau^*}(s)-h^{\pi_k}(s)\right)
\\
&\leq
\frac{D_{\pi_k}^{\pi_\tau^*}(s)-\rho^{-1}D_{\pi_{k+1}}^{\pi_\tau^*}(s)}{\eta}
+\frac{\eta}{2\lambda}
\left\|Q_\tau^k-Q_\tau^{\pi_k}\right\|_\infty^2
+
\left\langle\pi_{k+1}(\cdot\mid s)-\pi_k(\cdot\mid s),
Q_\tau^{\pi_k}(s,\cdot)\right\rangle
-\tau\!\left(h^{\pi_{k+1}}(s)-h^{\pi_k}(s)\right)
.
\label{eq:three-point-and-strong-convexity}
\end{align}
On the other hand, Lemma~\ref{lem:three-point-descent} with $p=\pi_k(\cdot\mid s)$ and the strong convexity of $h$ imply that
\begin{equation}
\left\langle\pi_{k+1}(\cdot\mid s)-\pi_k(\cdot\mid s),
Q_\tau^{\pi_k}(s,\cdot)\right\rangle
-\tau\!\left(h^{\pi_{k+1}}(s)-h^{\pi_k}(s)\right)
+\frac{\eta}{2\lambda}
\left\|Q_\tau^k-Q_\tau^{\pi_k}\right\|_\infty^2
\geq0.
\label{eq:shifted-policy-improvement-nonnegative}
\end{equation}
By restoring the state argument and leveraging the componentwise inequality $d_{d_\mu^*}^{\pi_{k+1}} \ge (1 - \gamma) d_\mu^*$, combining the non-negativity in~\eqref{eq:shifted-policy-improvement-nonnegative} with Lemma~\ref{lem:pdl} for $(\pi_{k+1}, \pi_k)$ with initial distribution $d_\mu^*$ yields
\begin{align}
\mathbb E_{s\sim d_\mu^*}
\left[\left\langle\pi_{k+1}(\cdot\mid s)-\pi_k(\cdot\mid s),
Q_\tau^{\pi_k}(s,\cdot)\right\rangle
-\tau\!\left(h^{\pi_{k+1}}(s)-h^{\pi_k}(s)\right)\right]
&\leq
V_\tau^{\pi_{k+1}}(d_\mu^*)-V_\tau^{\pi_k}(d_\mu^*)
\nonumber\\
&\quad+
\frac{\gamma\eta}{2\lambda(1-\gamma)}
\left\|Q_\tau^k-Q_\tau^{\pi_k}\right\|_\infty^2.
\label{eq:occupancy-transfer-policy-improvement}
\end{align}
Combining \eqref{eq:value-gap-pdl-critic-split}, \eqref{eq:three-point-and-strong-convexity}, and \eqref{eq:occupancy-transfer-policy-improvement}, we obtain
\begin{align}
V_\tau^*(\mu)-V_\tau^{\pi_k}(\mu)
&\leq
\frac{D_{\pi_k}^{\pi_\tau^*}(d_\mu^*)-\rho^{-1}D_{\pi_{k+1}}^{\pi_\tau^*}(d_\mu^*)}
{\eta(1-\gamma)}
+\frac{V_\tau^{\pi_{k+1}}(d_\mu^*)-V_\tau^{\pi_k}(d_\mu^*)}
{1-\gamma}\nonumber\\
&\quad+
\frac{\eta}{2\lambda(1-\gamma)^2}
\left\|Q_\tau^k-Q_\tau^{\pi_k}\right\|_\infty^2
+\frac{\mathbb E_{s\sim d_\mu^*}[B_k^{(k)}(s)]}{1-\gamma}.
\label{eq:one-step-weighted-value-gap}
\end{align}
Moreover, Lemmas~\ref{lem:bounded_value} and \ref{lem:bounded-stochastic-critic-iterates} imply that
\begin{equation}
\left\|Q_\tau^k-Q_\tau^{\pi_k}\right\|_\infty
\leq\frac{2(1+\tau\gamma H_h)}{1-\gamma}.
\label{eq:uniform-critic-error-bound}
\end{equation}
Multiply \eqref{eq:one-step-weighted-value-gap} by $\rho^{-k}$, sum over $k=0,\ldots,K-1$, and take expectation.  The Bregman terms telescope as
\[
\sum_{k=0}^{K-1}\rho^{-k}
\left(D_{\pi_k}^{\pi_\tau^*}(d_\mu^*)-\rho^{-1}D_{\pi_{k+1}}^{\pi_\tau^*}(d_\mu^*)\right)
\leq D_{\pi_0}^{\pi_\tau^*}(d_\mu^*).
\]
Since
$0\leq V_\tau^*(d_\mu^*)-V_\tau^{\pi_k}(d_\mu^*)
\leq2(1+\tau H_h)/(1-\gamma)$, summation by parts gives
\begin{align}
\sum_{k=0}^{K-1}\rho^{-k}
\left(V_\tau^{\pi_{k+1}}(d_\mu^*)-V_\tau^{\pi_k}(d_\mu^*)\right)
&\leq V_\tau^*(d_\mu^*)-V_\tau^{\pi_0}(d_\mu^*)
+
\sum_{k=1}^{K-1}(\rho^{-k}-\rho^{-(k-1)})
\left(V_\tau^*(d_\mu^*)-V_\tau^{\pi_k}(d_\mu^*)\right)
\nonumber\\
&\leq
\frac{2(1+\tau H_h)}{1-\gamma}
\left[1+\sum_{k=1}^{K-1}(\rho^{-k}-\rho^{-(k-1)})\right]
=\frac{2\rho^{-(K-1)}(1+\tau H_h)}{1-\gamma}.
\label{eq:weighted-value-increment-bound}
\end{align}
Finally, multiplying the resulting inequality by $(\rho^{-1}-1)/(\rho^{-K}-1)$, using $\rho^{-1}-1=\eta\tau$, \eqref{eq:uniform-critic-error-bound}, and the definition of $\widehat K$ proves \eqref{eq:weighted-value-gap-master-bound}.
\end{proof}

\vspace{1em}

We now decompose the last term in \eqref{eq:weighted-value-gap-master-bound}, following the notation of \citet[Lemma~3.9]{Li2026Markov-TDPMD}. Stack $\{\pi(\cdot\mid s)\}_{s\in\calS}$ into a vector $\pi\in\mathbb R^{|\calS||\calA|}$.  For each $s\in\calS$, let $E_s\in\mathbb R^{|\calA|\times|\calS||\calA|}$ select the coordinates associated with $s$, and set
\begin{equation}
\label{eq:state-index-matrix}
E_sQ=Q(s,\cdot),
\qquad
E_s\pi=\pi(\cdot\mid s),
\qquad
J_s:=E_s^\top E_s.
\end{equation}
Suppose $\alpha_k=\alpha\in(0,1]$ for every $k\geq0$. For a constant critic stepsize $\alpha$  and a policy $\pi$, define
\begin{equation}
A^\pi
:=
I-\alpha\Sigma_b
\left(I-\gamma P_{\scriptscriptstyle\calS\times\calA}^\pi\right).
\label{eq:critic-propagation-operator}
\end{equation}
Since $Q_\tau^{\pi_j}=\mathcal F_\tau^{\pi_j}Q_\tau^{\pi_j}$ and $\mathcal F_\tau^{\pi_j}$ has a linear part $\gamma P_{\scriptscriptstyle\calS\times\calA}^{\pi_j}$, it follows from \eqref{eq:expected-TD-PMD-critic-update} that
\begin{equation}
\label{eq:critic-error-propagation}
Q_\tau^{\pi_j}-Q_\tau^j
=
A^{\pi_j}\left(Q_\tau^{\pi_j}-Q_\tau^{j-1}\right)
-
\alpha\bar\omega_{j-1},
\qquad j\geq1.
\end{equation}
Fix $k\geq0$ and $s\in\calS$.  For $0\leq j\leq k$, define
\begin{equation}
\label{eq:Bjk-definition}
B_j^{(k)}(s)
:=
\left\langle
\left[\left(A^{\pi_j}\right)^{k-j}\right]^\top
J_s(\pi_\tau^*-\pi_j),
Q_\tau^{\pi_j}-Q_\tau^j
\right\rangle,
\qquad \left(A^{\pi_j}\right)^0:=I.
\end{equation}
For $1\leq j\leq k$, define
\begin{align}
C_j^{(k)}(s)
&:=
\left\langle
\left[\left(A^{\pi_j}\right)^{k-j+1}\right]^\top
J_s(\pi_\tau^*-\pi_j),
Q_\tau^{\pi_j}-Q_\tau^{\pi_{j-1}}
\right\rangle,
\label{eq:Cjk-definition}\\
D_j^{(k)}(s)
&:=
\left\langle
\left[\left(A^{\pi_j}\right)^{k-j+1}\right]^\top
J_s(\pi_{j-1}-\pi_j),
Q_\tau^{\pi_{j-1}}-Q_\tau^{j-1}
\right\rangle,
\label{eq:Djk-definition}\\
E_j^{(k)}(s)
&:=
\left\langle
\left(
\left[\left(A^{\pi_j}\right)^{k-j+1}\right]^\top
-
\left[\left(A^{\pi_{j-1}}\right)^{k-j+1}\right]^\top
\right)
J_s(\pi_\tau^*-\pi_{j-1}),
Q_\tau^{\pi_{j-1}}-Q_\tau^{j-1}
\right\rangle,
\label{eq:Ejk-definition}\\
F_j^{(k)}(s)
&:=
-
\left\langle
\left[\left(A^{\pi_j}\right)^{k-j}\right]^\top
J_s(\pi_\tau^*-\pi_j),
\alpha\bar\omega_{j-1}
\right\rangle.
\label{eq:Fjk-definition}
\end{align}
Here $B_j^{(k)}$ is the propagated signed critic error, with $B_0^{(k)}$ representing the initial remainder.  The terms $C_j^{(k)}$, $D_j^{(k)}$, $E_j^{(k)}$, and $F_j^{(k)}$ capture, respectively, the true-$Q$ drift, policy drift, operator drift, and Markov-noise contribution.

\begin{lemma}[Five-term decomposition]
\label{lem:five-term-signed-critic-decomposition}
One has 
\begin{align}
\left\langle
\pi_\tau^*(\cdot\mid s)-\pi_k(\cdot\mid s),
Q_\tau^{\pi_k}(s,\cdot)-Q_\tau^k(s,\cdot)
\right\rangle
&=
B_0^{(k)}(s)
+
\sum_{j=1}^k
\left(
C_j^{(k)}(s)
+
D_j^{(k)}(s)
+
E_j^{(k)}(s)
+
F_j^{(k)}(s)
\right).
\label{eq:BCDEF-pathwise-decomposition}
\end{align}
Consequently,
\begin{equation}
\mathbb E\left[B_k^{(k)}(s)\right]
\leq \mathbb E\left[\left|B_0^{(k)}(s)\right|\right]
+\sum_{j=1}^k\mathbb E\left[\left|C_j^{(k)}(s)\right|+\left|D_j^{(k)}(s)\right|\right]
+\left|\mathbb E\left[\sum_{j=1}^kE_j^{(k)}(s)\right]\right|
+\left|\mathbb E\left[\sum_{j=1}^kF_j^{(k)}(s)\right]\right|.
\label{eq:BCDEF-expectation-bound}
\end{equation}
\end{lemma}

\begin{proof}
By \eqref{eq:state-index-matrix} and
$\left(A^{\pi_k}\right)^0=I$,
\begin{align}
\left\langle
\pi_\tau^*(\cdot\mid s)-\pi_k(\cdot\mid s),
Q_\tau^{\pi_k}(s,\cdot)-Q_\tau^k(s,\cdot)
\right\rangle
&=
B_k^{(k)}(s)
=
B_0^{(k)}(s)
+
\sum_{j=1}^k
\left(
B_j^{(k)}(s)-B_{j-1}^{(k)}(s)
\right).
\label{eq:Bjk-telescoping-identity}
\end{align}
For $1\leq j\leq k$, using \eqref{eq:critic-error-propagation} we obtain
\begin{align}
B_j^{(k)}(s)
&=
\left\langle
\left[\left(A^{\pi_j}\right)^{k-j+1}\right]^\top
J_s(\pi_\tau^*-\pi_j),
Q_\tau^{\pi_j}-Q_\tau^{j-1}
\right\rangle
-
\left\langle
\left[\left(A^{\pi_j}\right)^{k-j}\right]^\top
J_s(\pi_\tau^*-\pi_j),
\alpha\bar\omega_{j-1}
\right\rangle
\nonumber\\
&=
B_{j-1}^{(k)}(s)
+
C_j^{(k)}(s)
+
D_j^{(k)}(s)
+
E_j^{(k)}(s)
+
F_j^{(k)}(s).
\label{eq:Bjk-increment-BCDEF}
\end{align}
Substituting \eqref{eq:Bjk-increment-BCDEF} into \eqref{eq:Bjk-telescoping-identity} proves \eqref{eq:BCDEF-pathwise-decomposition}.  Taking expectations and bounding the aggregate operator drift and Markov bias in absolute value proves \eqref{eq:BCDEF-expectation-bound}.
\end{proof}

\vspace{1em}

To proceed, let $q_b:=1-\alpha(1-\gamma)\widetilde\sigma_b\in(0,1)$. The definition of $A^{\pi_j}$ gives
\[
\|\left(A^{\pi_j}\right)^\ell x\|_\infty
\leq q_b^\ell\|x\|_\infty,
\qquad
\sum_{\ell=1}^\infty q_b^\ell
\leq\frac{1}{\alpha(1-\gamma)\widetilde{\sigma}_b}.
\]

\begin{lemma}[Bound for the true-\texorpdfstring{$Q$}{Q} drift]
\label{lem:C-drift-bound}
For every $k\geq1$ and $s\in\calS$,
\begin{equation}
\sum_{j=1}^k\left|C_j^{(k)}(s)\right|
\leq
\frac{2\gamma\eta}
{\lambda\alpha(1-\gamma)^2\widetilde{\sigma}_b}
\left(
\frac{1+\tau\gamma H_h}{1-\gamma}+\tau L_h
\right)^2.
\label{eq:C-drift-bound}
\end{equation}
\end{lemma}

\begin{proof}
The $\ell_1$--$\ell_\infty$ duality gives that
\begin{equation}
\left|C_j^{(k)}(s)\right|
\leq\|J_s(\pi_\tau^*-\pi_j)\|_1
\left\|\left(A^{\pi_j}\right)^{k-j+1}
(Q_\tau^{\pi_j}-Q_\tau^{\pi_{j-1}})\right\|_\infty
\leq2q_b^{k-j+1}\left\|Q_\tau^{\pi_j}-Q_\tau^{\pi_{j-1}}\right\|_\infty.
\label{eq:C-drift-one-step-first}
\end{equation}
Here $\|J_s(\pi_\tau^*-\pi_j)\|_1\leq2$, and the last inequality uses the contraction of $A^{\pi_j}$.  Lemma~\ref{lem:trajectory-regularity} then implies
\begin{equation}
\left|C_j^{(k)}(s)\right|
\leq
\frac{2\gamma\eta}{\lambda(1-\gamma)}
q_b^{k-j+1}
\left(
\frac{1+\tau\gamma H_h}{1-\gamma}+\tau L_h
\right)^2.
\label{eq:C-drift-one-step-final}
\end{equation}
Summing over $j$ and using $\sum_{j=1}^kq_b^{k-j+1} \leq[\alpha(1-\gamma)\widetilde\sigma_b]^{-1}$ proves \eqref{eq:C-drift-bound}.
\end{proof}

\begin{lemma}[Bound for the policy drift]
\label{lem:D-drift-bound}
For every $k\geq1$ and $s\in\calS$,
\begin{equation}
\sum_{j=1}^k\left|D_j^{(k)}(s)\right|
\leq
\frac{2\eta(1+\tau\gamma H_h)}
{\lambda\alpha(1-\gamma)^2\widetilde{\sigma}_b}
\left(
\frac{1+\tau\gamma H_h}{1-\gamma}+\tau L_h
\right).
\label{eq:D-drift-bound}
\end{equation}
\end{lemma}

\begin{proof}
By the $\ell_1$--$\ell_\infty$ duality and the contraction of $A^{\pi_j}$,
\begin{align}
\left|D_j^{(k)}(s)\right|
&\leq
\|J_s(\pi_{j-1}-\pi_j)\|_1
\left\|\left(A^{\pi_j}\right)^{k-j+1}
(Q_\tau^{\pi_{j-1}}-Q_\tau^{j-1})\right\|_\infty
\nonumber\\
&\leq
q_b^{k-j+1}
\|\pi_j-\pi_{j-1}\|_{1,\infty}
\left\|Q_\tau^{\pi_{j-1}}-Q_\tau^{j-1}\right\|_\infty.
\label{eq:D-drift-one-step-first}
\end{align}
Applying \eqref{eq:uniform-policy-increment-value-section} and \eqref{eq:uniform-critic-error-bound} yields
\begin{equation}
\left|D_j^{(k)}(s)\right|
\leq
\frac{2\eta(1+\tau\gamma H_h)}{\lambda(1-\gamma)}
q_b^{k-j+1}
\left(
\frac{1+\tau\gamma H_h}{1-\gamma}+\tau L_h
\right).
\label{eq:D-drift-one-step-final}
\end{equation}
Summing this series as in the preceding proof gives \eqref{eq:D-drift-bound}.
\end{proof}

\begin{lemma}[Exponentially weighted operator drift]
\label{lem:weighted-operator-drift}
Suppose $\rho^{-1}q_b<1$.  Then, for every $K\geq2$ and $s\in\calS$,
\begin{equation}
\left|\mathbb E\left[\sum_{j=1}^{\widehat K}E_j^{(\widehat K)}(s)\right]\right|
\leq
\frac{4\alpha\gamma\widetilde{\sigma}_b\eta(1+\tau\gamma H_h)}
{\lambda(1-\gamma)(1-\rho^{-1}q_b)^2}
\left(\frac{1+\tau\gamma H_h}{1-\gamma}+\tau L_h\right).
\label{eq:weighted-operator-drift-exact}
\end{equation}
\end{lemma}

\begin{proof}
For $N\geq1$, let $S_{j,N}:=\sum_{n=1}^N\rho^{-n} \bigl[(A^{\pi_j})^n-(A^{\pi_{j-1}})^n\bigr]$. Recall the distribution of $\widehat K$. Exchanging the finite $k$- and $j$-sums and setting $n=k-j+1$ give
\begin{equation}
\mathbb E\left[\sum_{j=1}^{\widehat K}E_j^{(\widehat K)}(s)\right]
=\frac{\rho^{-1}-1}{\rho^{-K}-1}\sum_{j=1}^{K-1}\rho^{-(j-1)}\mathbb E\left[\left\langle S_{j,K-j}^\top J_s(\pi_\tau^*-\pi_{j-1}),Q_\tau^{\pi_{j-1}}-Q_\tau^{j-1}\right\rangle\right].
\label{eq:weighted-operator-drift-reindexing}
\end{equation}
Thus it remains to bound $\|S_{j,N}\|_\infty$ uniformly in $N$. To this end, we first verify the following estimate whenever $\rho^{-1}q_b<1$,
\begin{equation}
\left\|
\left(I-\rho^{-1}A^\pi\right)^{-1}\alpha\Sigma_b
\right\|_\infty
\leq
\frac{\alpha\widetilde\sigma_b}{1-\rho^{-1}q_b}.
\label{eq:weighted-resolvent-sigma-bound}
\end{equation}
With $c:=\alpha\widetilde{\sigma}_b/(1-\rho^{-1}q_b)$, the identity
$A^{\pi_j}\mathbf 1
=\mathbf 1-\alpha(1-\gamma)\Sigma_b\mathbf 1$
implies $(I-\rho^{-1}A^{\pi_j})c\mathbf 1 \geq\alpha\Sigma_b\mathbf 1$. Since $A^{\pi_j}$ is nonnegative and $\rho^{-1}q_b<1$, comparison with its convergent Neumann series proves \eqref{eq:weighted-resolvent-sigma-bound}.  It also implies, for every $L\geq0$,
\begin{equation*}
\left\|
\sum_{\ell=0}^L
\left(\rho^{-1}A^{\pi_j}\right)^\ell\alpha\Sigma_b
\right\|_\infty
\leq
\frac{\alpha\widetilde{\sigma}_b}{1-\rho^{-1}q_b}.
\end{equation*}
Using
$A^{\pi_j}-A^{\pi_{j-1}}
=\alpha\gamma\Sigma_b
(P_{\scriptscriptstyle\calS\times\calA}^{\pi_j}
-P_{\scriptscriptstyle\calS\times\calA}^{\pi_{j-1}})$
and the power-difference identity,
\begin{align}
S_{j,N}
&=
\sum_{n=1}^N\rho^{-n}\sum_{m=0}^{n-1}
(A^{\pi_j})^{n-1-m}
(A^{\pi_j}-A^{\pi_{j-1}})(A^{\pi_{j-1}})^m
\nonumber\\
&=
\rho^{-1}\gamma
\sum_{m=0}^{N-1}
\left[
\sum_{\ell=0}^{N-1-m}
(\rho^{-1} A^{\pi_j})^\ell
\alpha\Sigma_b
\right]
\left(
P_{\scriptscriptstyle\calS\times\calA}^{\pi_j}
-P_{\scriptscriptstyle\calS\times\calA}^{\pi_{j-1}}
\right)
(\rho^{-1} A^{\pi_{j-1}})^m.
\label{eq:weighted-power-difference-expansion}
\end{align} 
The second equality substitutes the expression for $A^{\pi_j}-A^{\pi_{j-1}}$, sets $\ell=n-1-m$, and exchanges the two finite sums. For every vector $x$,
\[
\left\|(P_{\scriptscriptstyle\calS\times\calA}^{\pi_j}
-P_{\scriptscriptstyle\calS\times\calA}^{\pi_{j-1}})x\right\|_\infty
\leq\|\pi_j-\pi_{j-1}\|_{1,\infty}\|x\|_\infty,
\qquad
\|(\rho^{-1}A^{\pi_{j-1}})^m\|_\infty
\leq(\rho^{-1}q_b)^m.
\]
Therefore, applying \eqref{eq:weighted-resolvent-sigma-bound} to the first bracket in \eqref{eq:weighted-power-difference-expansion} and summing $\sum_{m\geq0}(\rho^{-1}q_b)^m=(1-\rho^{-1}q_b)^{-1}$ yield
\begin{equation}
\|S_{j,N}\|_\infty
\leq
\frac{
\rho^{-1}\alpha\gamma\widetilde{\sigma}_b
}{
(1-\rho^{-1}q_b)^2
}
\|\pi_j-\pi_{j-1}\|_{1,\infty}.
\label{eq:weighted-power-difference-bound}
\end{equation}
Finally, as $\|J_s(\pi_\tau^*-\pi_{j-1})\|_1\leq2$, \eqref{eq:uniform-critic-error-bound} implies that
\begin{align*}
\left|\mathbb E\left[\sum_{j=1}^{\widehat K}E_j^{(\widehat K)}(s)\right]\right|
&\leq
\frac{4(\rho^{-1}-1)(1+\tau\gamma H_h)}
{(1-\gamma)(\rho^{-K}-1)}
\sum_{j=1}^{K-1}\rho^{-(j-1)}
\mathbb E\left[\|S_{j,K-j}\|_\infty\right]
\\
&\leq
\frac{4\alpha\gamma\widetilde{\sigma}_b\eta(1+\tau\gamma H_h)}
{\lambda(1-\gamma)(1-\rho^{-1}q_b)^2}
\left(\frac{1+\tau\gamma H_h}{1-\gamma}+\tau L_h\right),
\end{align*}
where we leverage \eqref{eq:uniform-policy-increment-value-section}, \eqref{eq:weighted-power-difference-bound}, and the fact that $(\rho^{-1}-1)\sum_{j=1}^{K-1}\rho^{-(j-1)}/(\rho^{-K}-1)\leq\rho$ in the last inequality.
\end{proof}

\vspace{1em}

For the weights in \eqref{eq:exponential-batch-weights-value-section}, the condition $0\leq\vartheta<\kappa_b$ and summation of the resulting powers give
\begin{equation}
\sum_{t=0}^{B-1}c_t^k\kappa_b^t
=
\frac{\kappa_b^B-\vartheta^B}
{(\kappa_b-\vartheta)\sum_{\ell=0}^{B-1}\vartheta^\ell}
\leq\frac{\kappa_b^B}{\kappa_b-\vartheta}.
\label{eq:exponential-batch-mixing-bound}
\end{equation}

\begin{lemma}[Bound for the Markov-noise term]
\label{lem:F-noise-bound}
For every $k\geq1$ and $s\in\calS$,
\begin{equation}
\left|
\mathbb E\left[
\sum_{j=1}^kF_j^{(k)}(s)
\right]
\right|
\leq
\frac{4m_b(1+\tau\gamma H_h)}
{(1-\gamma)^2\widetilde{\sigma}_b}
\frac{\kappa_b^B}{\kappa_b-\vartheta}.
\label{eq:F-noise-bound}
\end{equation}
\end{lemma}

\begin{proof}
Since $A^{\pi_j}$, $\pi_j$, and the initial state of batch $j-1$ are $\mathcal H_{j-1}$-measurable, conditioning the definition of $F_j^{(k)}(s)$ on $\mathcal H_{j-1}$ yields
\begin{align}
\left|\mathbb E\left[F_j^{(k)}(s)\right]\right|
&=
\alpha\left|\mathbb E\left[
\left\langle
\left[\left(A^{\pi_j}\right)^{k-j}\right]^\top
J_s(\pi_\tau^*-\pi_j),
\mathbb E[\bar\omega_{j-1}\mid\mathcal H_{j-1}]
\right\rangle\right]\right|
\nonumber\\
&\leq
2\alpha q_b^{k-j}
\mathbb E\left[
\left\|\mathbb E[\bar\omega_{j-1}\mid\mathcal H_{j-1}]
\right\|_\infty\right]
\nonumber\\
&\leq
\frac{4\alpha m_b(1+\tau\gamma H_h)}{1-\gamma}
q_b^{k-j}
\frac{\kappa_b^B}{\kappa_b-\vartheta},
\label{eq:F-noise-one-step-bound}
\end{align}
where we use $\|J_s(\pi_\tau^*-\pi_j)\|_1\leq2$ and the contraction of $A^{\pi_j}$ in the first inequality, and the second inequality is due to Lemma~\ref{lem:stochastic-error-bound} together with \eqref{eq:exponential-batch-mixing-bound}.  Finally, linearity of expectation and the triangle inequality give
\begin{equation*}
\left|\mathbb E\left[\sum_{j=1}^kF_j^{(k)}(s)\right]\right|
\leq\sum_{j=1}^k\left|\mathbb E[F_j^{(k)}(s)]\right|
\leq\frac{4m_b(1+\tau\gamma H_h)}{1-\gamma}
\frac{\kappa_b^B}{\kappa_b-\vartheta}
\alpha\sum_{\ell=0}^{k-1}q_b^\ell.
\end{equation*}
Since $\alpha\sum_{\ell=0}^{k-1}q_b^\ell \leq[(1-\gamma)\widetilde\sigma_b]^{-1}$, \eqref{eq:F-noise-bound} is proved.
\end{proof}

\vspace{1em}

We now substitute the preceding five bounds above into \eqref{eq:weighted-value-gap-master-bound}.  First, the initial propagation term satisfies, uniformly in $s\in\calS$,
\begin{equation}
\left|B_0^{(k)}(s)\right|
\leq2q_b^k\left\|Q_\tau^{\pi_0}-Q_\tau^0\right\|_\infty.
\label{eq:initial-propagation-bound}
\end{equation}
Indeed, $\|J_s(\pi_\tau^*-\pi_0)\|_1\leq2$ and $\|\left(A^{\pi_0}\right)^kx\|_\infty \leq q_b^k\|x\|_\infty$.  Hence, whenever $\rho^{-1}q_b<1$,
\begin{equation}
\mathbb E\left[\left|B_0^{(\widehat K)}(s)\right|\right]
\leq
\frac{2(\rho^{-1}-1)\left[1-(\rho^{-1}q_b)^K\right]}
{(\rho^{-K}-1)(1-\rho^{-1}q_b)}
\left\|Q_\tau^{\pi_0}-Q_\tau^0\right\|_\infty.
\label{eq:weighted-initial-propagation-bound}
\end{equation}

\begin{proof}[Proof of Theorem~\ref{thm:off-policy-markov-value-gap}]
Since $\rho=(1+\eta\tau)^{-1}\in(0,1)$,
\[
\rho^{-K}\geq2
\quad\Longleftrightarrow\quad
K\geq\frac{\log 2}{\log(1/\rho)}.
\]
Thus, the stated integer lower bound on $K$ ensures $\rho^{-K}\geq2$. Writing $a:=\alpha(1-\gamma)\widetilde{\sigma}_b$, the stepsize bound in the theorem gives $1-\rho=\eta\tau/(1+\eta\tau)\leq a/2$ and hence the common stability margin
\[
1-\rho^{-1}q_b
=
\rho^{-1}\bigl[a-(1-\rho)\bigr]
\geq
\frac{1}{2}\rho^{-1}\alpha(1-\gamma)\widetilde{\sigma}_b>0.
\]
Thus, Lemma~\ref{lem:weighted-operator-drift} applies. Moreover, Lemmas~\ref{lem:bounded_value} and \ref{lem:bounded-stochastic-critic-iterates}, together with $Q_\tau^0=0$, give $\|Q_\tau^{\pi_0}-Q_\tau^0\|_\infty \leq(1+\tau\gamma H_h)/(1-\gamma)$. Consequently, \eqref{eq:weighted-initial-propagation-bound} bounds the contribution of $B_0$ by the second summand in $C_1$.  The condition $\rho^{-K}\geq2$ also gives
\[
\frac{2(\rho^{-1}-1)\rho^{-(K-1)}(1+\tau H_h)}
{(1-\gamma)^2(\rho^{-K}-1)}
\leq
\frac{4\eta\tau(1+\tau H_h)}{(1-\gamma)^2}.
\]
Substituting Lemmas~\ref{lem:C-drift-bound}, \ref{lem:D-drift-bound}, \ref{lem:weighted-operator-drift}, and \ref{lem:F-noise-bound} into \eqref{eq:BCDEF-expectation-bound} and then into \eqref{eq:weighted-value-gap-master-bound}, the $C$- and $E$-terms combine as
\[
\frac{2\gamma\eta
\bigl[1+\tau\gamma H_h+\tau(1-\gamma)L_h\bigr]
\bigl[9(1+\tau\gamma H_h)+\tau(1-\gamma)L_h\bigr]}
{\lambda\alpha(1-\gamma)^5\widetilde{\sigma}_b}.
\]
Collecting the remaining terms proves \eqref{eq:explicit-three-term-value-gap}.
\end{proof}

\subsection{Proof of Corollary~\ref{cor:off-policy-markov-value-sample-complexity}}

\paragraph{Constant-stepsize bias.}
For this proof only, set $A:=1+\tau\gamma H_h$, $R:=\tau(1-\gamma)L_h$, and $G:=(A+R)(9A+R)$.  Since $A\geq1$ and $R\geq0$, we have $G\geq9A^2\geq9$.  Additionally, the constant $C_2$ is at most
\[
\frac{4\bigl[1+\lambda\tau(1+\tau H_h)\bigr]G}
{\lambda\alpha(1-\gamma)^5\widetilde{\sigma}_b}
\]
as $\alpha(1-\gamma)\widetilde{\sigma}_b<1$ and $\gamma\leq1$. Indeed, after division by $G/[\lambda\alpha(1-\gamma)^5\widetilde{\sigma}_b]$, the first three summands are bounded by $4\lambda\tau(1+\tau H_h)$, $1$, and $1$, respectively, while the last is bounded by $2$. Thus, \eqref{eq:epsilon-policy-stepsize} ensures that the second term in \eqref{eq:explicit-three-term-value-gap} does not exceed $\epsilon/3$.

\paragraph{Stability and optimization error.}
Writing $a:=\alpha(1-\gamma)\widetilde{\sigma}_b\in(0,1)$ within this proof, there holds
\[
\frac{\eta_\epsilon\tau(2-a)}{a}
\leq
\frac{2\lambda\tau}
{12\bigl[1+\lambda\tau(1+\tau H_h)\bigr]G}
<1.
\]
Hence, the stepsize condition in Theorem~\ref{thm:off-policy-markov-value-gap} holds and $0<\eta_\epsilon\tau<1$, which implies $\log(1+\eta_\epsilon\tau)\geq\eta_\epsilon\tau/2$. Furthermore,
\[
\frac{\tau D_{\pi_0}^{\pi_\tau^*}(d_\mu^*)}{1-\gamma}
+
\frac{4\eta_\epsilon\tau(1+\tau\gamma H_h)}
{\alpha(1-\gamma)^3\widetilde{\sigma}_b}
\leq
\frac{\tau D_{\pi_0}^{\pi_\tau^*}(d_\mu^*)}{1-\gamma}
+
\frac{4(1+\tau\gamma H_h)}{(1-\gamma)^2}.
\]
Consequently, \eqref{eq:epsilon-outer-iterations} yields
\[
(1+\eta_\epsilon\tau)^{K_\epsilon}-1
\geq
\max\!\left\{
1,\frac{3}{\epsilon}
\left[
\frac{\tau D_{\pi_0}^{\pi_\tau^*}(d_\mu^*)}{1-\gamma}
+
\frac{4(1+\tau\gamma H_h)}{(1-\gamma)^2}
\right]\right\}.
\]
In particular, $(1+\eta_\epsilon\tau)^{K_\epsilon}\geq2$, and the first term in \eqref{eq:explicit-three-term-value-gap} does not exceed $\epsilon/3$.

\paragraph{Mixing error and sample count.}
Finally, \eqref{eq:epsilon-batch-length} gives
\[
\frac{\kappa_b^{B_\epsilon}}{\kappa_b-\vartheta}
\leq
\frac{\epsilon(1-\gamma)^3\widetilde{\sigma}_b}
{12m_b(1+\tau\gamma H_h)},
\]
so the third term does not exceed $\epsilon/3$.  Summing the three bounds proves \eqref{eq:epsilon-value-guarantee}.  Moreover, the definition of $L_h$ gives
\[
R
=
\tau(1-\gamma)H_h
+
\max\left\{
\frac{\tau(1-\gamma)}{2}
\max_{s\in\calS,\,a\in\calA}
D_h\!\left(e_a\,\middle\|\,\pi_0(\cdot\mid s)\right),
A+\tau(1-\gamma)H_h
\right\},
\]
so $A$, $R$, and $G$ are bounded when $\tau,H_h$, and $\max_{s,a}D_h(e_a\,\|\,\pi_0(\cdot\mid s))$ are fixed.  Hence \eqref{eq:epsilon-policy-stepsize}--\eqref{eq:epsilon-batch-length} give $K_\epsilon=\widetilde{\mathcal O}\bigl(
[(1-\gamma)^5\widetilde\sigma_b\epsilon]^{-1}\bigr)$ and $B_\epsilon=\widetilde{\mathcal O}(1)$ under the stated convention, and then the claimed bound for $N_\epsilon=K_\epsilon B_\epsilon$ follows.

\end{document}